\documentclass{article}

\usepackage{microtype}
\usepackage{graphicx}
\usepackage{subfigure}
\usepackage{booktabs} 
\usepackage{tikz}

\usepackage{hyperref}

\usepackage[accepted]{icml2024}

\usepackage{caption}
\usepackage{listings}

\usepackage{amsmath}
\usepackage{amssymb}
\usepackage{mathtools}
\usepackage{amsthm}

\usepackage[capitalize,noabbrev]{cleveref}
\crefname{equation}{Eq.}{Eqs.}
\Crefname{equation}{Equation}{Equations}

\crefname{section}{Sec.}{Secs.}
\Crefname{section}{Section}{Sections}

\crefname{appendix}{App.}{Apps.}
\Crefname{appendix}{Appendix}{Appendixes}

\newcommand{\Lp}[2]{\left\Vert #2 \right\Vert_#1^#1}

\theoremstyle{plain}
\newtheorem{theorem}{Theorem}[section]

\newtheorem{lemma}[theorem]{Lemma}

\theoremstyle{definition}

\theoremstyle{remark}

\usepackage[textsize=tiny]{todonotes}

\icmltitlerunning{Decoupled Weight Decay for Any $p$ Norm}

\begin{document}

\twocolumn[
\icmltitle{Decoupled Weight Decay for Any $p$ Norm}




\begin{icmlauthorlist}
    \icmlauthor{Nadav Joseph Outmezguine}{ucb,lbl}
    \icmlauthor{Noam Levi}{epfl,tau}
\end{icmlauthorlist}

\icmlaffiliation{tau}{
Raymond and Beverly Sackler School of Physics and Astronomy, Tel-Aviv University,
Tel-Aviv 69978, Israel}
\icmlaffiliation{epfl}{\'Ecole Polytechnique F\'ed\'erale de Lausanne~(EPFL), Switzerland.}
\icmlaffiliation{ucb}{Berkeley Center for Theoretical Physics, University of California, Berkeley, CA 94720, USA}
\icmlaffiliation{lbl}{Theory Group, Lawrence Berkeley National Laboratory, Berkeley, CA 94720, USA}

\icmlcorrespondingauthor{Nadav J. Outmezguine}{NJO@Berkeley.edu}
\icmlcorrespondingauthor{Noam Levi}{noam.levi@epfl.ch, noam@mail.tau.ac.il}

\icmlkeywords{Sparsity, Regularization, Neural Networks, Optimization, Weight Decay}

\vskip 0.3in
]



\printAffiliationsAndNotice{}  

\begin{abstract}
With the success of deep neural networks (NNs) in a variety of domains, the computational and storage requirements for training and deploying large NNs have become a bottleneck for further improvements. Sparsification has consequently emerged as a leading approach to tackle these issues. In this work, we consider a simple yet effective approach to sparsification, based on the Bridge, or $L_p$ regularization during training. We introduce a novel weight decay scheme, which generalizes the standard $L_2$ weight decay to any $p$ norm. We show that this scheme is compatible with adaptive optimizers, and avoids the gradient divergence associated with $0<p<1$ norms. We empirically demonstrate that it leads to highly sparse networks, while maintaining generalization performance comparable to standard $L_2$ regularization.
\end{abstract}

\section{Introduction}
\label{sec:intro}

Deep neural networks (NNs) have garnered unparalleled success across a variety of domains ranging from vision \citep{rn50} to language \citep{transformer,wavenet, wavernn}. Modern network performance has been shown to scale with both model complexity and dataset size, 
now operating in the jointly large parameter and large data size regime~\citep{deep-learning-scaling}. The resources required to train and deploy large NNs can, consequently, impose a bottleneck on further improvements~\citep{kaplan2020scaling}. 
For instance, Inception-V4~\citep{szegedy2016inceptionv4}, requires 16 billion arithmetic operations and 43 million parameters to be evaluated, while GPT-4~\citep{gpt4} requires over 1.75 trillion parameters (2 TiB assuming 16 bits per parameter).
Furthermore, training such models becomes increasingly expensive. Large language models (LLMs) already require supercomputers for training, with costs potentially reaching tens of millions of dollars per run, as cited in GPT-3~\citep{gpt-3}. Moreover, these models induce tremendous energy costs, as highlighted in the study on energy costs~\citep{energy_cost}. It is therefore critical to study {\it sparsification} during the training process as an avenue to manage resources during training and deployment~\citep{book_tibshirani}.

We define the sparsity of an NN as the fraction of its parameters that have a value of exactly zero. Higher sparsity therefore corresponds to fewer informative parameters, and thus, potentially, lower computational and storage requirements.
With zero valued weights, any multiplications (which dominate neural network computation) can be skipped, and sufficiently sparse models can be stored and transmitted compactly using sparse matrix formats. 
Sparse models are required to store more information per parameter relative to their denser counterparts. They may, therefore, be less prone to overfitting, and exhibit better generalization performance~\citep[e.g.][]{NIPS1989_6c9882bb,NIPS1992_303ed4c6,248452,hoefler2021sparsity}.
It has been empirically shown that deep NNs can perform effectively even with high levels of sparsity~\citep{lwac, exploring-sparsity-rnn, sws, gromov2024unreasonable}. This property is leveraged to reduce costs and enable the deployment of state-of-the-art models in resource-constrained environments \citep{fisher-pruning, wavernn, lpcnet}. In particular, modern GPU architectures like NVIDIA's Ampere, equipped with Sparse Tensor Cores, can leverage unstructured sparsity at levels as low as 50\% to achieve significant inference speedups \citep{mishra2021accelerating}. Additionally, recent research has demonstrated that applying sparsity in the fine-tuning of large language models can lead to substantial inference acceleration on both CPUs and GPUs, without compromising accuracy \citep{kurtic2023sparse}.

In recent years, various techniques for inducing sparsity in NNs have been proposed, including post-training pruning and dynamical regularization-based approaches~\citep{kwon2022a,lasby2024dynamic,yin2023dynamic}. 
Our work falls in the latter category, focusing in particular on weight regularization.
Weight regularization methods methods introduce a penalty term (regularizer) into the loss function to constrain the magnitude of each of the model parameters. 
This constraint implicitly biases the network towards sparser, rather than denser model representations and gradually reduces the magnitudes of the network weights during the training process.
Generally, regularization methods can be written as ${\cal L}'(\boldsymbol{x},\boldsymbol{w}) = {\cal L}(\boldsymbol{x},\boldsymbol{w}) + R(\boldsymbol{w})$. Here, ${\cal L}$ is the original loss function defined on the weights $\boldsymbol{w}$ and the data samples $\boldsymbol{x}$, while $R$ is the regularizer term which acts only on the weights.

The most common weight regularization method is $L_2$. 
While $L_2$ regularization achieves smaller weights and better generalization error at the end of the training process~\citep[e.g.][]{Plaut1986,Nowlan1992SimplifyingNN,NIPS1991_8eefcfdf,NIPS1991_d64a340b,NEURIPS2019_8744cf92}, it does not result in a sparse network representation. This is since the penalty term is 'rotationally invariant', meaning that it does not favor any particular direction in the weight space.
A ubiquitous regularization method which does result in sparse networks is $L_1$, or Lasso regularization~\cite{tibshirani1996regression}. Elastic-net regularization, which combines both $L_1$ and $L_2$ norms of the weights, was suggested as a method that exhibits both the sparsity of $L_1$ and the generalization performance of $L_2$~\cite{zou2005regularization}.

Towards a more general form of regularization, \citet{frank1993statistical} proposed Bridge regularization, or $L_p$ regularization, $R_p\sim\Lp{p}{\boldsymbol{w}}$, in which $p$ is chosen based on the problem at hand. {\it This is the underlying regularization method which we base our work upon}.
Bridge regression enjoys some desirable statistical properties, such as sparsity and near-unbiased estimates for $L_p$ norms in the range $p \in (0, 1)$ \citep{fan2001variable,Lp_quasi_norm}. 
Importantly, when $p < 1$, the so-called $p$-norm does not adhere to the triangle inequality. It does, however, satisfy the weaker condition $\|x + y\| \leq C\times(\|x\| + \|y\|)$ for some $C > 1$, which qualifies it as a quasi-norm. Additionally, in this range of $p < 1$, the quasi-norm is non-convex, making it more challenging to optimize~\citep{9ccd776b-aa6b-37a6-8f20-ee38e65b5919}. Previous works suggested variations on Bayesian sampling approaches to bypass these issues for $p \in (0,1)$~\citep{polson2014bayesian,loria2023suretuned}. Despite these complexities, for ease of notation, we refer to it as the $p$-norm throughout this paper.

In this work, we introduce a straightforward implementation of $L_p$ regularization. This method maintains the key benefits of $L_p$ regularization, such as sparsification and generalization, while {\it avoiding numerical instabilities caused by exploding gradients at the small weights limit}. Our approach integrates a single, simple step that complements any modern optimizer with minimal computational overhead. Furthermore, it can easily be adapted to more flexible regularization schemes, including variants of the Elastic Net.

Our main contributions are as follows:
\begin{itemize}
    \item 
    In \cref{sec:equivalence}, we illustrate how an $L_{p<2}$ regularized problem with $N$ parameters is equivalent to another optimization problem with $N$ additional auxiliary parameters. We show that the optimal solutions of both problems coincide, and for $p<2$ these solutions are expected to be sparse.
    \item 
    In \cref{sec:pwd}, we introduce our main contribution, the `$p$-norm Weight Decay' ($p$WD), a novel weight decay scheme for any $p$-norm regularization. We use a toy example to demonstrate that, across all $0 \leq p$ values, $p$WD avoids gradient instabilities and stabilizes training dynamics. We then present the $p$WD algorithm which implements this new weight decay method.
    \item 
    In \cref{sec:empirical}, we empirically assess the performance of $p$-norm Weight Decay ($p$WD) across various tasks and architectures, including comparisons with other sparsification methods. Our results show that $p$WD achieves high levels of sparsity while maintaining excellent network generalization.
    \item 
    In~\cref{sec:dynamical_pwd}, we discuss some limitations of $p$WD, suggest possible extensions, and propose future research directions.
\end{itemize}

\section{Related Work}
\label{sec:related}

\par{\bf Regularization and sparsification:} 
Besides linear regression, Bridge regularization has been applied to support vector machines \citep{liu2007support}, giving impressive results. As a special case of
Bridge regularization, $L_{1/2}$ has been shown to exhibit useful
statistical properties including sparseness and unbiasedness~\citep{xu20101}. 
Different training algorithms have been proposed for
training neural networks with $L_{1/2}$ weight penalty~\cite{fan2014convergence,yang2018l1}.
In terms of Bayesian estimation, Ridge and Lasso penalties imply a Gaussian and Laplacian prior on model weights, respectively. On the other hand, an $L_p$ penalty corresponds to the Generalized Gaussian prior on the model weights~\cite{frank1993statistical}.

\par{\bf Proximal operators:}
The proximal operator for the (lasso) regularization, known as the soft thresholding operator, is widely used in the literature~\cite{daubechies2003iterative}. The proximal operator for various other specific values of norms has also been studied in the literature, for example in \cite{l1over2,proximal_lp_small}, but result in cumbersome schemes. Partially for that reason, approximated operators were devised, for example in \cite{proximal_lp_inexact}.

\par{\bf Bridge regression:}
First suggested by~\cite{frank1993statistical}, Bridge regression, or $L_p$ regularization, has been studied extensively. 
It has been shown that Bridge regularization performs better than Ridge, Lasso and elastic-net in certain regression problems~\citep{park2011bridge}. 
In recent years, works such as \cite{polson2012bayesian, khan2018bridgeout} consider stochastic variations, while \citet{mcculloch2023sparse} integrate the concept of $L_p$ regularization for subset selection with constitutive NNs to obtain sparse networks, and \cite{doi:10.1080/03610926.2022.2028841} consider an adaptive re-weighting method.
Of special importance is~\cite{TOH2023119505}, which proposes an analytic solution for Bridge regression based on solving a penalized error formulation
using a proximal operator approach, closely in line with this work.

\section{Equivalent Formulation of $L_p$ Regularization}
\label{sec:equivalence}
Our starting point is the optimization problem of minimizing the empirical risk, or loss function ${\cal L}(\boldsymbol{w})$, with respect to the weights $\boldsymbol{w}\in\mathbb{R}^{N_w}$, where $N_w$ is the total number of weights (including biases), subject to an $L_p$ regularization term, $R_p(\boldsymbol{w})=(\lambda_p/p)\Lp{p}{\boldsymbol{w}}$, where $p>0, \lambda_p\in\mathbb{R}^+$ and $\|\cdot\|_p$ is the $p$-norm. 
In this section we introduce a higher dimensional dual optimization problem, where the loss is regularized instead by
\begin{align}
    \label{eq:reg_dual}
    R_p(\boldsymbol{w},\boldsymbol{s})=\frac{\lambda_p}{2}\sum_i\left[s_iw_i^2+K(s_i)\right],
\end{align}
where $w_i ,s_i \in \mathbb{R}, i=1,\ldots,N_w$. 
Here $K(s_i)$ is a function of $s_i$ only, chosen such that the two regularization terms satisfy the equality $R_p(\boldsymbol{w})=\min_{\boldsymbol{s}}R_p(\boldsymbol{w},\boldsymbol{s})$. Specifically, for $p\neq 2$, one suitable choice for $K(s_i)$ is given by
\begin{align}
    \label{eq:g}
    K(s_i)=\frac{2-p}{p}s_i^{p/(p-2)},
\end{align}  
under the restriction that $s_i>0$. 
In~\cref{app:equivalence}, we prove formally that the extended optimization problem is equivalent to the original one, in the sense that they share the same global and local minima.
By design, the minimum of the original optimization problem coincides with that of the extended one, namely,
\begin{align}
    \label{eq:opt}
    \min_{\boldsymbol{w}}{\cal L}(\boldsymbol{w})+\lambda_p R_p(\boldsymbol{w})=\min_{\boldsymbol{w},\boldsymbol{s}}{\cal L}(\boldsymbol{w})+\lambda_p R_p(\boldsymbol{w},\boldsymbol{s}).  
\end{align}

Before we move on, it is important to note that  $R_{p}(\boldsymbol{w},\boldsymbol{s})$ is non-convex\footnote{This can very easily be seen in the $w\to \infty$ limit, where $R_p$ is given by $\lambda_p s w^2$, independent of $K$, and the hessian has eigenvalues $\pm w$.}. Therefore, even for $p\geq 1$ the extended optimization problem is non-convex.

The regularizer $R_p(w,s)$ is what is known as a {\it biconvex function}~\cite{biconvex}. 
In simple terms, it is a convex function of $w$ for any fixed $s>0$ and vice versa. Biconvex functions can exhibit multiple local minima. Nevertheless, we refer the reader again to~\cref{app:equivalence} for a formal proof that both local and global minima of the original and extended optimization problems coincide.
Specifically, for $p<1$, the generalized loss  ${\cal L}(\boldsymbol{w})+\lambda_p R_p(\boldsymbol{w},\boldsymbol{s})$ will exhibit local minima at $w_i=0$, where the global minimum of $R_p(w_i,s_i)$ is located.
This is to be expected since it is also the case in the original formulation of the $p<1$ norm. 

\section{$p$-norm Weight Decay}
\label{sec:pwd}

Having established that the empirical risk minimization problem with $L_p$ regularization can be expressed as a biconvex optimization problem, we now turn to an important implication of this formulation, which is based on the Alternate Convex Search (ACS) algorithm. 
Alternate Convex Search is a common strategy for optimizing biconvex functions, in which we alternate between optimizing with respect to one variable while keeping the other fixed. In each step, standard convex optimization techniques can be used, and the problem is guaranteed to converge to a (possibly local) minimum~\cite{biconvex}. Building upon this approach, in this section we derive a weight decay step, analogous to traditional $L_2$ weight decay, but extended to any $p$ norm.  We refer to this method as $p$-norm Weight Decay ($p$WD).

\subsection{Proximal Operator Representation}
\label{sec:proximal}
As mentioned above, we will follow the ACS approach, where we optimize over $\boldsymbol{w}$ and $\boldsymbol{s}$ sequentially. In each $\boldsymbol{s}$ update step we set it to its optimal value, 
\begin{align}
    \label{eq:s_opt}
    \boldsymbol{s}_n=\vert \boldsymbol{w}_n\vert^{p-2},
\end{align}
where the absolute value is taken element-wise, and the subscript $n$ denotes the $n$-th iteration. Since we are dealing with sparsity inducing regularization, we expect some weights to vanish. For $p<2$, we see that the $s_i$'s corresponding to vanishing weights will diverge.

Moving now to optimize over $\boldsymbol{w}$, following the ACS approach we hold $\boldsymbol{s}$ fixed at last value $\boldsymbol{s}_n$. Note that $K(\boldsymbol{s}_n)$ is also fixed and does not impact the optimization with respect to $\boldsymbol{w}$. Therefore, we can effectively optimize ${\cal L}+(\lambda_p/2)\sum_i s_i w_i^2$. This second term is seemingly a standard $L_2$ regularization term, however,  the possible divergence of some of the $s_i$'s calls for a subtle treatment

From this point on, we shall focus on gradient based optimization methods, as they are most commonly used when training NNs. 
Taking any gradient based approach, we will have to include the gradient of $R_p$ with respect to $\boldsymbol{w}$, which is given by
\begin{align}
    \label{eq:grad_w}
    \boldsymbol{\nabla}_{\boldsymbol{w}}R_p(\boldsymbol{w},\boldsymbol{s}_n)=\lambda_p\boldsymbol{s}_n\circ\boldsymbol{w}= \lambda_p\vert \boldsymbol{w}_n\vert^{p-2}\circ\boldsymbol{w},
\end{align}
where ``$\circ$'' denotes element-wise multiplication. For~$p<1$, this gradient will become very large for small weights, rendering any finite learning rate approach unstable. Taking instead a decoupled weight decay approach~\citep{loshchilov2019decoupled}, we are tempted to write the $\boldsymbol{w}$ update step as
\begin{align}
    \label{eq:decoupled}
    \boldsymbol{w}\gets\left(\boldsymbol{w}-\alpha\delta\boldsymbol{w}\right)\circ(\boldsymbol{1}-\alpha\lambda_p\boldsymbol{s}_n),
\end{align}
where $\delta\boldsymbol{w}$ is either the gradient of $\cal L$, or any other adaptive step based on the unregularized loss ${\cal L}$, and $\alpha \in \mathbb{R}^+$ is the learning rate. This update rule, however, still does not ensure stability; weights that the regularization term drives to zero will be multiplied by a divergent negative weight decay factor, giving rise to an oscillatory behavior around 0. 

To overcome this instability, we propose to use the proximal operator of $R_p(\boldsymbol{w},\boldsymbol{s}_n)$ with respect to $\boldsymbol{w}$ at fixed $\boldsymbol{s}_n$. We review the basics of proximal operators in~\cref{app:proximal}, where we also derive the proximal gradient step for $R_p$,
\begin{align}
    \label{eq:proximal}
    \boxed{\boldsymbol{w}\gets\frac{\boldsymbol{w}-\alpha\delta\boldsymbol{w}}{\boldsymbol{1}+\alpha\lambda_p\boldsymbol{s}_n}= \frac{\boldsymbol{w}-\alpha\delta\boldsymbol{w}}{\boldsymbol{1}+\alpha\lambda_p\vert \boldsymbol{w}_n\vert^{p-2}}\,.\,}
\end{align}
In the equation above, division is carried out element-wise. For $\alpha\lambda_p{s}_n\ll 1$, this is equivalent to the decoupled weight decay step in~\cref{eq:decoupled}. However, here, the numerator is always larger than 1, driving the weights to $0$ for~$\alpha\lambda_p{s}_n\gg1$,~as desired. As seen in~\cref{eq:proximal}, $w_{n,i}=0$ is a fixed point of the proximal gradient step for all $p<2$.
The stability of this fixed point is discussed in~\cref{app:stability}, where we find it to be stable only for $p<1$, similar to the original problem.

\cref{eq:proximal} is the main result of this work, and we refer to this method as {\it $p$-norm Weight Decay ($p$WD)}.

\subsection{Toy Example}
\label{sec:example}
\begin{figure}[t]
    \vskip 0.2in
    \begin{center}
    \centerline{\includegraphics[width=0.5\textwidth]{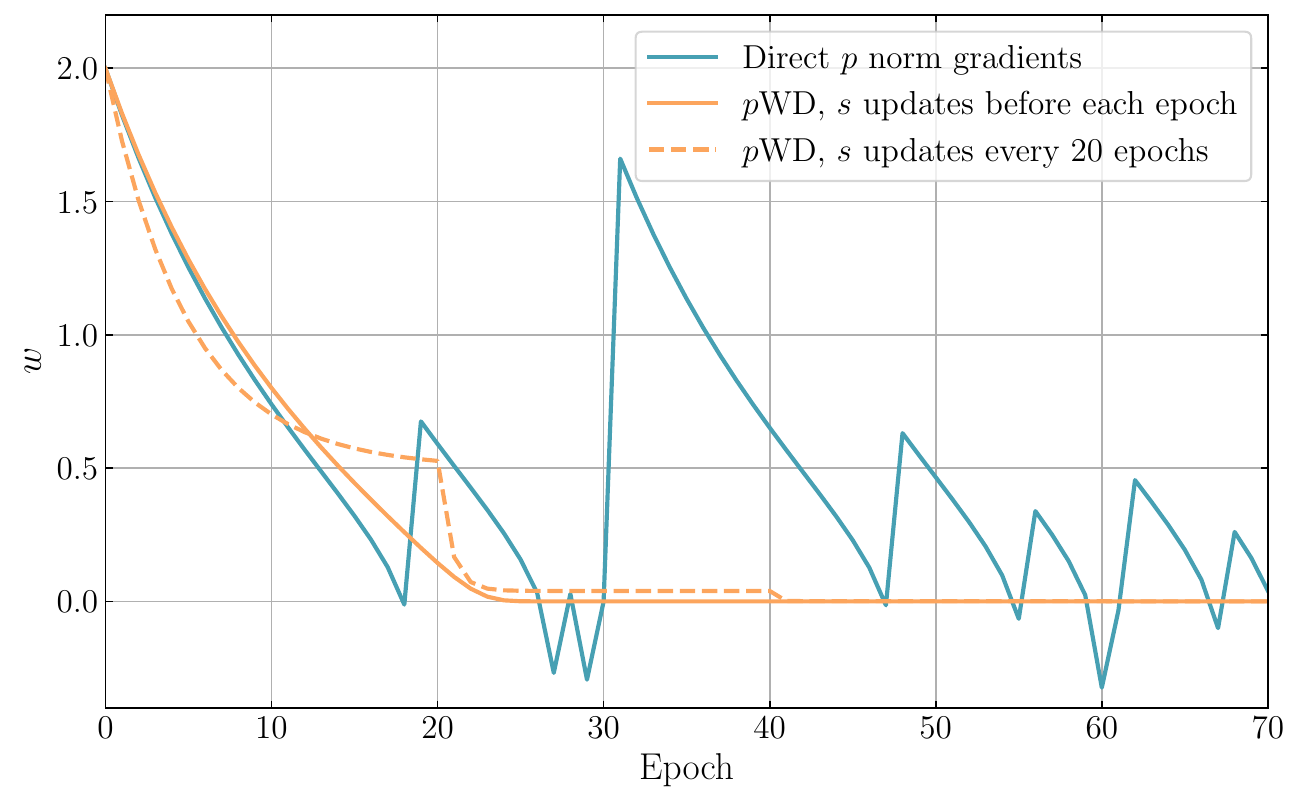}}
    \caption{Toy example of weight evolution under gradient descent for the loss ${\cal L}=(w-1)^2/2+\Lp{p}{w}/p$. \textbf{Dotted line:} represents simple gradient descent where the norm is added directly to the gradient. The weight fails to converge to 0 due to the exploding gradient of the $p$-norm near 0. The \textbf{dashed line} represents the evolution of the weight under the update rule in~\cref{eq:proximal}, where we update $s$ every 20 $w$ steps.  The \textbf{solid line} represents the evolution of the weight under the update rule in~\cref{eq:proximal}, where we update $s$ at every $w$ step.
    The latter is an implementation of {\it $p$-norm Weight Decay ($p$WD)}.
    We see that in both implementations of our method, the weight converges smoothly to 0.}
    \label{fig:toy_example}
    \end{center}
    \vskip -0.2in
\end{figure}
To demonstrate the challenges of $L_p$ regularization, and the benefits of our approach, we start with a simple toy example. 
Consider the single variable regularized loss function 
\begin{align}
\label{eq:toy_loss}
    {\cal L}(w)=
    \frac{1}{2}(w-1)^2+ 
    \frac{\lambda_p}{p}\vert w\vert^p,
    \qquad
    w\in\mathbb{R}.
\end{align}
For any $\lambda_p>0$ and $p>0$, the minimum of ${\cal L}(w)$ lies on the segment $w\in[0,1]$. For $\lambda_p\geq1$ and $p\leq1$, the minimum is found at $w=0$. 
For simplicity, we consider simple gradient descent as the update rule for $w$, leading to the equation
\begin{align}
    \label{eq:grad}
    w_{t+1}=w_t-\alpha \left[w_t\left(1+\lambda_p\vert w_t\vert^{p-2}\right) -1\right].
\end{align}
For $p<1$, and $\lambda_p>1$, the regularized loss gradient will drive the weight to $w=0$, leading to a divergent gradient. 
We demonstrate this in~\cref{fig:toy_example}, where the evolution of the weight under gradient descent for $p=0.6$, $\lambda_p=1$, and a learning rate $\alpha=0.1$ is shown as a cyan line. The weight starts flowing towards $0$, until reaching a point where the $p$-norm gradient becomes too large and the weight changes sign, leading to an oscillatory behavior around $0$. 
On the contrary, the orange lines represent the evolution of the weight under the update rule in~\cref{eq:proximal}. For the solid line we update $s_n=\vert w\vert^{-1.4}$ before each $w_n$ step, which results in a smooth convergence to $0$. 
For the dashed line we initialize $s_0=0.1$ and update $s_n=\vert w\vert^{-1.4}$ once every 20 $w$ steps. We see that the weight converges smoothly to $0$ in a step-wise pattern, without any oscillations.
In the remainder of this work, we adopt the smoothest approach and update $s$ at every $w$ step.

\subsection{The $p$WD Algorithm}
\label{sec:pwd_alg}

\begin{algorithm}
    \caption{Gradient Based $p$WD}
    \label{alg:pwd}
    \begin{algorithmic}[1]
        \STATE
        {\bf given}
        initial learning rate $\alpha \in \mathbb{R}^+$, weight decay regularization factor $\lambda_p \in \mathbb{R}^+$ weight norm number $p\in\mathbb{R}^+$, gradient-based optimization algorithm and its hyper-parameters 
            \STATE 
        {\bf initialize}
        time step $t \gets 0$, parameter vector $\boldsymbol{w}_{t=0} \in \mathbb{R}^n$, schedule multiplier $\eta \in \mathbb{R}^+$
        \REPEAT
            \STATE $t \gets t + 1$
            \STATE $g_t\gets \boldsymbol{\nabla} \mathcal{L}(\boldsymbol{w}_{t-1})$
            \STATE $\delta\boldsymbol{w}_t\gets$OptimizerWeightUpdate($g_t,t$)
            \STATE $\eta_t \gets $ SetScheduleMultiplier($t$) 
            \STATE $\tilde{\boldsymbol{w}}_t \gets\boldsymbol{w}_{t-1} - \eta_{t}\alpha\delta\boldsymbol{w}_t$
            \STATE
            The $p$WD Step:  
            \\
            $$
            \boxed{\boldsymbol{w} \gets
            \frac{\vert \boldsymbol{w}_{t-1}\vert^{2-p}}{\vert \boldsymbol{w}_{t-1}\vert^{2-p}+\eta_t \alpha\lambda_p}\tilde{\boldsymbol{w}}_{t}}$$
        \UNTIL{stopping criterion is met}
        \\
        \STATE \textbf{return} optimized parameters $\boldsymbol{w}_t$
    \end{algorithmic}
    \end{algorithm}

Based on the discussion above, we can now present our proposed $L_p$ weight decay algorithm,~\cref{alg:pwd}. To limit the scope of this work, we will focus only on the case where $\boldsymbol{s}$ is updated at every $\boldsymbol{w}$ step. We note, however, that the algorithm can be easily modified to update $\boldsymbol{s}$ every $n$ $\boldsymbol{w}$ steps, where $n$ is a hyper-parameter. This modification, along with few other variants, are discussed in~\cref{sec:dynamical_pwd}. In~\cref{app:padam_code}, we provide an example \texttt{PyTorch} implementation of $p$WD based on the Adam optimizer.

Lines~$1-8$ of~\cref{alg:pwd} are the usual steps for any gradient based optimizer, such as SGD~\citep{sgd}, Adam~\citep{kingma2014adam}, RMSprop~\citep{tieleman2012lecture} etc., encapsulated by the function OptimizerWeightUpdate$(t,g_t)$, which may include momentum or higher moments. We have also explicitly included learning rate scheduling. 
The novelty appears at Line 9, where we impose the $p$WD weight decay step. This weight decay step is the same as the one in~\cref{eq:proximal}, assuming the auxiliary parameters $\boldsymbol{s}$ are set before every $\boldsymbol{w}$ step.

In~\cref{app:non_increasing}, we prove that gradient descent with the $p$WD step guarantees that the original loss function is non-increasing.
We note that Lines 8 and 9 are split into two steps for clarity of presentation, but in practice can be carried out simultaneously to avoid the memory overhead of storing $\tilde{\theta}_t$. 
As discussed in~\cref{sec:proximal,app:stability}, the $p$WD step has a fixed point at $w_i=0$ for all $p<2$. \textbf{Regardless of the stability of this fixed point, it is important to stress that a parameter initialized at $w_i=0$ will remain fixed to this value during training.} In the experiments presented in~\cref{sec:empirical}, we therefore avoid decaying parameter tensors that are initialized at $0$, such as biases and batch normalization parameters.

\section{Sparsity with $p$WD in Realistic Settings}
\label{sec:empirical}
\begin{figure*}[t]
    \vskip 0.2in
    \begin{center}
    \centerline{\includegraphics[width=\textwidth]{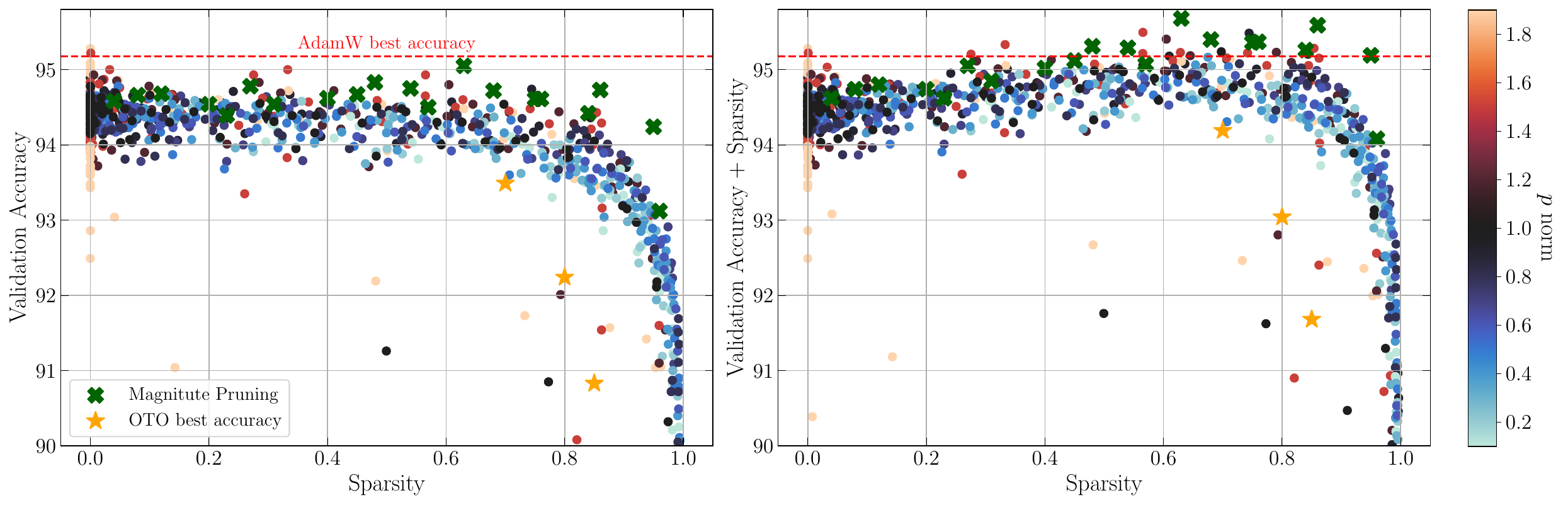}}
    \caption{
        Validation accuracy vs. sparsity for ResNet18 trained on CIFAR-10. Each point represents a different instance of the network trained for 100 epochs, with a different choice of $p$, $\lambda_p$, and learning rate $\alpha$. Points of different colors indicate different choices of $p$, optimizing over $\lambda_p, \alpha$. The {\bf dashed-red line} indicate the best accuracy achieved using AdamW. The {\bf orange stars} indicate the best accuracy runs obtained using {\it Only Train Once}~\citep[{\it OTO}\text{,}][]{OTO}. The {\bf green crosses} indicate the best accuracies obtained using {\it iterative magnitude pruning}.  \textbf{Left:} Validation accuracy vs. sparsity. \textbf{Right:} Example of the accuracy/sparsity trade-off given in~\cref{eq:tradeoff}.
    }
    \label{fig:accurecy_vs_sparsity}
    \end{center}
    \vskip -0.2in
\end{figure*}

\begin{figure*}[t]
    \vskip 0.2in
    \begin{center}
    \centerline{\includegraphics[width=\textwidth]{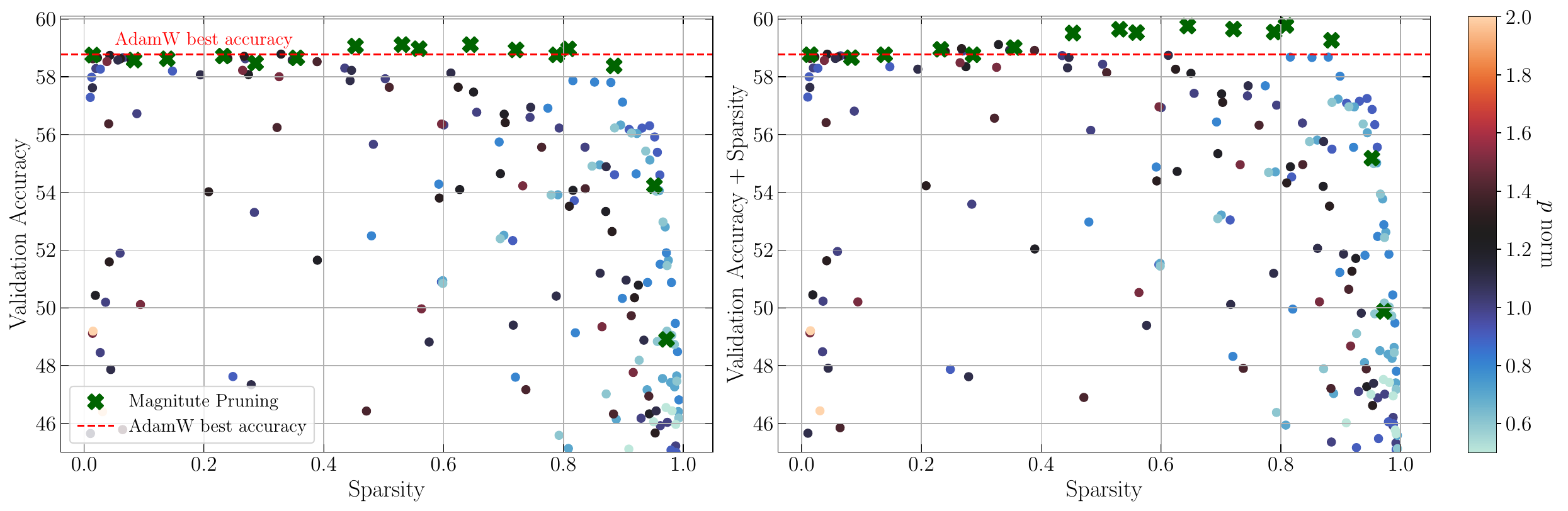}}
    \caption{
        Validation accuracy vs. sparsity for nanoGPT trained on Tiny Shakespeare. Each point represents a different instance of the network trained for 5000 iterations, with a different choice of $p$, $\lambda_p$, and learning rate $\alpha$. Points of different colors indicate different choices of $p$, optimizing over $\lambda_p, \alpha$. The dashed-red line indicates the best accuracy achieved using AdamW. The {\bf green crosses} indicate the best accuracies obtained using {\it iterative magnitude pruning}. \textbf{Left:} Validation accuracy vs. sparsity. \textbf{Right:} Example of the accuracy/sparsity trade-off given in~\cref{eq:tradeoff}.
    }
    \label{fig:accurecy_vs_sparsity_gpt}
    \end{center}
    \vskip -0.2in
\end{figure*}
In this section, we empirically test the performance of our $p$WD scheme. We use Adam~\cite{kingma2014adam} as our base optimizer, and supplement it with the proposed $p$WD step. 
We refer to this optimizer as $p$Adam. 
A simple \texttt{PyTorch} implementation is described in~\cref{app:padam_code}. Throughout this work, we keep the Adam hyper-parameters fixed at their default values. For all our experiments, we adopted a learning rate schedule that combines a linear warm-up phase with a subsequent cosine annealing. The precise experimental details, architectures, and hyper-parameters are given in~\cref{app:experimental_details}. 

The experiments were conducted on two standard datasets: CIFAR-10~\cite{cifar10} and Tiny Shakespeare~\cite{shake}. We employed two deep learning models, for vision and text, namely ResNet18~\cite{he2015deep}
and NanoGPT~\cite{2023-karpathy}; a character-level language based on GPT2~\cite{radford2019language}. We trained both architectures with a cross-entropy loss, and used sparsity and accuracy as our main validation matrices\footnote{While accuracy might not be the first choice for token level language model practitioners, we find it suitable and intuitive for our character level language model experiments.}.
We trained our models using  the $p$Adam optimizer for a range of $p$ values, and using AdamW~\cite{loshchilov2019decoupled} as a point of reference for performance. 
For each optimizer and each $p$ value choice, we scanned over both the learning rate and the weight decay $\lambda_p$. Importantly, we did not decay one dimensional parameter tensors, such as biases and batch normalization parameters. Such parameters are commonly initialized at $0$, which, as mentioned above, will remain there during training under $p$WD. At the same time, they constitute a small fraction of the total number of parameters, and thus do not significantly affect the overall sparsity level.

\subsection{Sparsity and Performance for $p$WD}
The main results of this paper are shown in~\cref{fig:accurecy_vs_sparsity,fig:accurecy_vs_sparsity_gpt}, where we show the validation accuracy against sparsity, with sparsity defined as the fraction of weights smaller than $10^{-13}$. It is evident that $p$Adam finds sparse network representations, with accuracy level gradually decreasing as sparsity increases. For ResNet18 trained on CIFAR10, we observe that models with sparsity as high as $99.5\%$ have achieved accuracy higher than $90\%$, significantly higher than random guessing, which gives 10\% accuracy.  The highest sparsity level we find for accuracy drop of less than $2\%$ relative to the best AdamW accuracy was $94.4\%$ for ResNet18, and $89.9\%$ for NanoGPT.

Different colors in~\cref{fig:accurecy_vs_sparsity,fig:accurecy_vs_sparsity_gpt} represent different $p$ norms. 
We find that the sparsest networks are obtained at values of $p<1$ while the best generalizing networks are found for $1<p<2$. See discussion in~\cref{sec:dynamical_pwd} for possible explanations and improvement strategies. 
These results imply that for a given dataset and architecture, an optimal $p$ may be found, under a choice of accuracy/sparsity trade-off. In the right panel of~\cref{fig:accurecy_vs_sparsity}, we show a concrete realization of this trade-off, whereby a loss of $1\%$ in accuracy is equal to a $100\%$ increase in sparsity, defined by
\begin{align}
    \label{eq:tradeoff}
    \text{Trade-off~Metric}=
    \mathrm{Val.~Acc.}[\%] + \mathrm{Sparsity}.
\end{align}
For ResNet18, the validation+Sparsity scatter exhibit a clear peak at around a sparsity of 80\%, indicating on the optimal $p\simeq 1.2$ under such trade-off. The nanoGPT results are less conclusive, and in general less well behaved, but still the sparsity+accuracy scatter seem to be roughly constant up to sparsity of $\sim90\%$, achieved by $p\geq0.8$

\subsection{Comparison with Other Methods}
\begin{table}[t]
    \caption{Accuracies above a given sparsity level for ResNet-18 on CIFAR-10.  Comparison of $p$WD to different sparsification methods.
    }
    \label{tab:results_resnet}
    \vskip 0.15in
    \begin{center}
    \begin{small}
    \begin{sc}
    \begin{tabular}{lcccc}
    \toprule
    sparsity    & 0\%       & 70\%      & 80\%      & 90\% \\
    \midrule
    MP          & 95.18\%   & 94.73\%   & 94.73\%   & 94.24\%\\
    OTO         & 95.18\%   & 93.49\%   & 92.24\%   & 87.82\%\\
    $p$WD       & 95.28\%   & 94.74\%   & 94.43\%   & 93.79\%\\
    \bottomrule
    \end{tabular}
    \end{sc}
    \end{small}
    \end{center}
    \vskip -0.1in
\end{table}
\label{sec:comparison}
Next, we compare the performance of $p$WD to other sparsification methods. We focus on two methods for sparsificaiton during training. The first is the {\it Only Train Once}~\citep[{\it OTO}\text{,}][]{OTO} method, which we apply only to the ResNet18 experiments. The second, which we apply to both ResNet18 and nanoGPT,  is a simple iterative magnitude pruning (MP) method. 
There, we simply set to 0 all weights smaller than a certain threshold, once every fixed number of iterations. For both MP and OTO, we use AdamW as the base optimizer, and scan over the learning rate, weight decay, and one pruning hyper-parameter (pruning threshold for MP and pruning fraction for OTO). For both methods we start pruning after $10\%$ of the training epochs, and increase the pruning threshold/fraction for the next $10\%$ of the epochs. The results are shown in~\cref{fig:accurecy_vs_sparsity} and in~\cref{tab:results_resnet,tab:results_gpt}. For ResNet18, we find that $p$WD outperforms OTO, but is inferior to MP until a sparsity of about $90\%$, where $p$WD retains higher accuracy. For nanoGPT, we find, much like for ResNet18, that MP outperforms $p$WD until a sparsity of about $95\%$, where $p$WD performs slightly better.

Overall, our findings indicate that $p$WD achieves results comparable to MP and surpasses the performance of OTO. It's important to note, however, that the MP approach involved an additional advantage: the use of AdamW combined with iterative setting of small weights to zero. This approach is akin to the Elastic Net~\citep{zou2005regularization}, where sparsity is induced by the $L_1$ term and optimization stabilization and generalization are aided by the $L_2$ term. In contrast, $p$WD employs a single regularization term, with the parameter $p$ effectively balancing sparsity and generalization performance\footnote{Since our goal was to highlight the utility of $p$WD alone, we did not combine it with other sparsification methods, but this is trivially done. In particular, combining iterative magnitude pruning and $p$WD is as simple as performing MP with regular WD.}.
In the subsequent section, we will explore how $p$WD can be elegantly extended to simultaneously enhance both these aspects.
As a simple demonstration of the potential of an extended $p$WD, we ran an additional nanoGPT experiment with $p=0.8$, this time using AdamW instead of Adam. We fixed the $L_2$ weight decay of $2\times 10^{-3}$ and scanned over the learning rate and the $p$WD weight decay $\lambda_{0.8}$. For sparsity of 90\% we obtained an accuracy of 57.15\%. Improving over the $p$Adam result and surpassing the MP result.

\begin{table}[t]
    \caption{Accuracies above a given sparsity level for nanoGPT on Tiny Shakespeare.  Comparison of $p$WD to different sparsification methods.}
    \label{tab:results_gpt}
    \vskip 0.15in
    \begin{center}
    \begin{small}
    \begin{sc}
    \begin{tabular}{lcccc}
    \toprule
    sparsity    & 0\%           & 80\%      & 90\%      & 95\% \\
    \midrule
    MP       & 59.11\%       & 58.97\%   & 57.07\%   & 54.23\%\\
    $p$WD       & 58.79\%       & 57.93\%   & 56.49\%   & 55.92\%\\
    \bottomrule
    \end{tabular}
    \end{sc}
    \end{small}
    \end{center}
    \vskip -0.1in
\end{table}

\section{Limitations of $p$WD and Possible Variations}
\label{sec:dynamical_pwd}
This paper, with the goal of establishing $p$WD as a viable $L_p$ regularization method, was focused on a specific implementation of $p$WD. In this implementation, $\boldsymbol{s}$ is updated to its optimal value, given in~\cref{eq:s_opt}, at every $\boldsymbol{w}$ update step. 
We identify two aspects of $p$WD that call for further investigation: The first is the existence of a fixed point at $w_i=0$ for all $p<2$, which is also a local minimum for $p<1$. The second aspect is the generalization performance of the resulting networks. In this section we discuss these two aspects, and propose possible variations of $p$WD that might improve the performance of the resulting networks. The common theme of these variations is that they all involve richer dynamics, which  comes with a price of increased complexity, and an increased number of hyper-parameters.

\subsection{Avoiding the $w_i=0$ Fixed Point}

The $w_i=0$ fixed point arises due to the large denominator in~\cref{eq:proximal} whenever $w_i\to0$. The disadvantage of this fixed point is especially important for parameters initialized at $w_i=0$, which will remain frozen during training. At the same time, the existence of this fixed point is crucial for the algorithm to converge to the sparse solutions it was designed to find. 
A successful algorithm will therefore avoid getting stuck at the $w_i=0$ local minimum, while still converging to $w_i=0$ when appropriate. Below we discuss a few possible approaches to achieve this goal.

\par{\bf $\boldsymbol{s}$ dynamics:} In the proposed $p$WD~\cref{alg:pwd}, $\boldsymbol{s}$ in~\cref{eq:proximal} is updated to its optimal value before every $\boldsymbol{w}$ step. Here, we suggest promoting $\boldsymbol{s}$ to a learnable parameter. From that perspective,  currently at each epoch $\boldsymbol{s}$ is updated according to~\cref{eq:s_opt}. Therefore, a weight initialized at $w_i=0$ will force $s_i\to\infty$, which will in turn force $w_i$ to remain at $0$. 
From the perspective of a dynamical system, we can say that $\boldsymbol{s}$ is a 'fast' variable, which reaches optimal value before the 'slow' variable, $\boldsymbol{w}$, has had enough time to update its value. By `slowing down' $\boldsymbol{s}$, we can avoid the $w=0$ fixed point. We can initialize $\boldsymbol{s}=\boldsymbol{1}$, which means that the network starts evolving under a standard weight decay. Then, during training we can let $\boldsymbol{s}$ evolve in a desired pace towards its optimal value. This can be done either by updating $\boldsymbol{s}$ every $n$ steps of $\boldsymbol{w}$ updates, or by applying an SGD-like update rule to $\boldsymbol{s}$. In this case, if a weight passes through $w_i=0$, it will be able to continue evolving as long as $s_i$ is not too large. We note that while this approach does come with a small memory overhead. The computational overhead is negligible as the gradients of $\boldsymbol{s}$ are trivial to compute and implement.

\par{\bf $p$ scheduling:} It is clear from~\cref{eq:proximal} that for $p=2$ there is no fixed point at $w_i=0$ and we revert to the standard weight decay scheme. 
Having the network start training with $p=2$ for a few epochs, and then gradually decreasing $p$ towards a desired $p<2$ value, would allow the network to avoid the $w_i=0$ fixed point at initialization. At the same time, the network will still be able to converge to $w_i=0$ at later stages as $p$ decreases. Further, restarting $p$ to $p=2$ and decreasing to a smaller value repeatedly, would allow the network to explore the parameter region around the $w_i=0$ fixed point, and possibly escape it. 

\subsection{Generalization Performance:} One important observation from~\cref{fig:accurecy_vs_sparsity} is that the best generalizing networks are found for $1<p<2$, while the most sparse networks are found for $p<1$. One possible explanation is that $p>1$ norms, such as AdamW, incur larger penalties on larger weights than smaller ones, in contrast to $p<1$. The importance of regularizing large weights is well known~\citep[for example][]{loshchilov2019decoupled}, and perhaps $1<p<2$ provide a better balance between the two. This hypothesis is also supported by the performance of MP with AdamW, which is essentially a combination of $p=2$ and $p=0$ penalties. 

Both $p$ scheduling and $\boldsymbol{s}$ dynamics, discussed above, can potentially achieve as similar effect even when $p<1$. 
For example, consider the case of $p$ scheduling with restarts. Upon each restart, larger weights will be penalized more harshly, while the smaller ones will instead be regularized towards the end of a cycle when $p$ is small. 
Alternatively, if we adopt slow $\boldsymbol{s}$ dynamics, large weights will be penalized until $s$ reaches its optimal value (which is small for large weights).

\par{\bf Elastic Weight Decay:} Another possible approach is to use a variation of elastic net proposed in~\cite{zou2005regularization}. In the original elastic net, the loss is regularized by a combination of $L_1$ and $L_2$ norms. Supposedly, this combination allows the network to benefit from the sparsity inducing properties of $L_1$ regularization, while still benefiting from the stability of $L_2$ regularization. Repeating the steps leading to~\cref{eq:proximal}, we can in principle add both an $L_1$ and an $L_2$ norm and  achieve the following elastic net weight decay step
\begin{align}
    \label{eq:elastic}
    \boldsymbol{w}\gets\frac{\boldsymbol{w}-\alpha\delta\boldsymbol{w}}{\boldsymbol{1}+\alpha\left(\lambda_1\vert\boldsymbol{w}\vert^{-1}+\lambda_2\right)}\,.
\end{align}
Moreover, the flexibility of our proposed $p$WD allows us to generalize the elastic net approach to any combination of $L_{p<2}$ norms, $\sum_p\lambda_p\Lp{p}{\boldsymbol{w}}$. In which case, the weight decay step becomes
\begin{align}
    \label{eq:elastic_general}
    \boldsymbol{w}\gets\frac{\boldsymbol{w}-\alpha\delta\boldsymbol{w}}{\boldsymbol{1}+\alpha\sum_p\lambda_p\vert\boldsymbol{w}\vert^{p-2}}\,.
\end{align}
While our derivation of~\cref{eq:proximal} relied on the specific construction as presented in~\cref{eq:reg_dual,eq:g}, which is valid only for $p<2$, we see no reason why~\cref{eq:elastic_general} should not be valid for any $p>0$\footnote{In the case of simple SGD, the $w\neq0$ fixed point of $p$WD is given by the $\boldsymbol{w}$ solving  $0=\nabla{\cal L}(\boldsymbol{w})+\lambda_p|\boldsymbol{w}|^{p-2}\circ\boldsymbol{w}$. This is precisely the equation for the minimum of ${\cal L}(\boldsymbol{w})+(\lambda_p/p)\Lp{p}{\boldsymbol{w}}$, assuming ${\cal L}$ is convex, regardless of the value of $p$.}.
In principle, optimizing with a combination of one $p<1$ norm and one $p>1$ norm, might provide a better balance between the two, and improve the generalization performance of the resulting network. 
As mentioned in the previous section, a verification of this prediction was tested on nanoGPT which essentially combined $p=2$ and $p=0.8$ weight decay. The results were superior to both MP and a single $p$ implementation of $p$WD. 

\section{Conclusions}
\label{sec:conclusions}

In this work, building upon the works of \citet{frank1993statistical} and \citet{loshchilov2019decoupled}, we developed a novel regularization-based sparsification scheme, which we dubbed {\it $p$-norm Weight Decay}. Our method, a proximal approximation of $L_p$ regularization, dynamically drives weights to zero during training within a stable optimization framework. $p$WD is as simple to implement as any standard optimizer. It operates as a supplemental weight decay step and is, therefore, compatible with any modern optimizer. Additionally, it incurs negligible memory and computational overhead. Our ultimate goal is to incorporate $p$WD into popular deep learning frameworks such \texttt{PyTorch} and \texttt{Tensorflow}, therby to rendering sparse training as straightforward as using any modern optimizer.

Through empirical evaluation, we demonstrate that our approach enables performance gains and high levels of sparsity. Specifically, we are able to prune ResNet and NanoGPT models to extremely sparse configurations while retaining high accuracy. Our results clearly demonstrate that $p$WD provides an effective approach for network sparsification, competing with state-of-the-art methods in terms of maintaining accuracy while achieving highly sparse networks. 

This work is, however, only a first step towards unveiling the full potential of $p$WD. Iterative magnitude pruning, for example, outperformed some aspects of $p$WD in the experiments presented in~\cref{sec:empirical}. As discussed in~\cref{sec:dynamical_pwd}, we believe that the performance of $p$WD can be further improved by incorporating richer dynamics. We leave the thorough exploration of these dynamics to future work.

Going beyond the scope of NNs, $p$WD is essentially a noval gradient based approximated optimization approach for $p\leq 2$ norms. As such $p$WD can be implemented on many problems other than machine learning. One example may lie in Variational Quantum Circuits~\cite{Cerezo_2021},  where decreasing the number of parameters is desirable. 

In conclusion, the straightforward implementation, flexibility, and potentially adaptive nature of $p$WD, have promise to stimulate new areas of investigation into optimizing neural networks and automated architecture design.

\section*{Impact Statement}
The method suggested in this paper simplifies sparsification in neural networks training. Thereby, potentially making machine learning more efficient and accessible in environments with limited resources. By reducing energy and computational demands, our approach could have a wider impact, facilitating sustainable AI technology use across various sectors.

\section*{Acknowledgements}
We would like to thank Andrey Gromov for useful discussions and Ioannis Mavrothalassitis for assisting us to derive some of the convergence proofs. NJO acknowledges support from the National Science Foundation under the grant No. PHY-1915314. NJO further thanks B. Nachman for computing resources. NL would like to thank G-Research for the award of a research grant, as well as the CERN-TH department for their hospitality during various stages of this work.

\bibliography{LPbib}
\bibliographystyle{icml2024}

\appendix
\onecolumn

\section{Proof of the Extended Problem Equivalence}
\label{app:equivalence}
In this appendix, we prove the equivalence between the original and the extended optimization problems. First, we show that $R_p(\boldsymbol{w},\boldsymbol{s})$ in \cref{eq:reg_dual} with $K$ as in \cref{eq:g} satisfies $R_p(\boldsymbol{w})=\min_{\boldsymbol{s}}R_p(\boldsymbol{w},\boldsymbol{s})$, provided $s_i>0$ and $0<p<2$. We want to show that 
\begin{align}
    \label{eq:equivalence}
    R_p(\boldsymbol{w})=\frac{\lambda_p}{p}\sum_i \vert w_i\vert^p=\min_{\boldsymbol{s}>0}\frac{\lambda_p}{2}\sum_i\left[w_i^2s_i+\frac{2-p}{p}s^{p/(p-2)}\right]=\min_{\boldsymbol{s}>0}R_p(\boldsymbol{w},\boldsymbol{s})\,.
\end{align}
It is enough to show for a single component since the problem is separable. By taking second partial derivatives, the bi-convexity of $R_p(w,s>0)$ is established. The minimum of $R_p(w,s>0)$ at fixed $w$ is therefore unique and can be found by setting $\partial R_p(w,s>0)/\partial s=0$. This gives
\begin{align}
    \label{eq:s_opt_der}
    \frac{\partial R_p(w,s_*)}{\partial s_*}=\frac{\lambda_p}{2}\left[w^2-s_*^{2/(p-2)}\right]=0\;\;\Rightarrow\;\; s_*=\vert w\vert^{p-2}\;\;\Rightarrow\;\; R_p(w,s_*)=\frac{\lambda_p}{p}\vert w\vert^p\,.
\end{align}
This shows that the minimum of $R_p(w,s>0)$ is indeed $R_p(w)$, and therefore the equivalence in \cref{eq:equivalence} holds. For future reference, we note that the minimum of $R_p(w,s>0)$ at fixed $w$ is unique and a continuous function of $w$.

The equivalence of the optimization problem is stated in the following theorem.

\begin{theorem}
\label{thm:equivalence}
    Let ${\cal L}:\;\mathbb{R}^n\to\mathbb{R}$, let $R:\;\mathbb{R}^n\times\mathbb{R}^n\to\mathbb{R}$ be a smooth bi-convex function. Define $F(\boldsymbol{w},\boldsymbol{s})={\cal L}(\boldsymbol{w})+R(\boldsymbol{w},\boldsymbol{s})$, and $\hat{s}(\cdot)=\underset{\boldsymbol{s}}{\operatorname{argmin}}\;R(\cdot,\boldsymbol{s})$. Assuming $\hat{s}$ is continuous, the following holds:
    \item[(i)] Let $(\boldsymbol{w}^*, \boldsymbol{s}^*)$ be a local minimum of $F(\boldsymbol{w},\boldsymbol{s})$. Then, $\boldsymbol{w}^*$ is a local minimum of $F(\boldsymbol{w},\hat{s}(\boldsymbol{w}))$.
    \item[(ii)] Let $\boldsymbol{w}^*$ be a local minimum of $F(\boldsymbol{w},\hat{s}(\boldsymbol{w}))$. Then, $(\boldsymbol{w}^*, \boldsymbol{s}^*)$ is a local minimum of $F(\boldsymbol{w},\boldsymbol{s})$ with $\boldsymbol{s}^*=\hat{s}(\boldsymbol{w}^*)$.
    \item[(iii)] In particular, the global minimum of $F(\boldsymbol{w},\hat{s}(\boldsymbol{w}))$ is the same as the global minimum of $F(\boldsymbol{w},\boldsymbol{s})$ and occurs at the same value of $\boldsymbol{w}$.
\end{theorem}

\begin{proof}
    \item[(i)] Assume that $(\boldsymbol{w}^*, \boldsymbol{s}^*)$ is a local minimum of $F(\boldsymbol{w},\boldsymbol{s})$. By the biconvexity of $R$, we have that $F(\boldsymbol{w}^*,\boldsymbol{s})$ is convex in $\boldsymbol{s}$ for any fixed $\boldsymbol{w}^*$. Therefore, $\boldsymbol{s}^*=\hat{s}(\boldsymbol{w}^*)$ is the unique minimizer of $F(\boldsymbol{w}^*,\boldsymbol{s})$ for any fixed $\boldsymbol{w}^*$.  We therefore know that the there is a local minimum for $F(\boldsymbol{w},\hat{s}(\boldsymbol{w}))$ at $(\boldsymbol{w}^*, \boldsymbol{s}^*=\hat{s}(\boldsymbol{w}^*))$. Since this is a local minimum, $\exists \epsilon_w>0,\epsilon_s>0$ such that $\forall \boldsymbol{w},\boldsymbol{s}$ with $\Vert\boldsymbol{w}-\boldsymbol{w}^*\Vert<\epsilon_w$ and $\Vert\boldsymbol{s}-\boldsymbol{s}^*\Vert<\epsilon_s$ we have $F(\boldsymbol{w},\boldsymbol{s})\geq F(\boldsymbol{w}^*, \boldsymbol{s}^*)$. To prove that $\boldsymbol{w}^*$ is a local minimum of $F(\boldsymbol{w},\hat{s}(\boldsymbol{w}))$, we need to show that $\exists \epsilon>0$ such that $\forall \boldsymbol{w}$ with $\Vert\boldsymbol{w}-\boldsymbol{w}^*\Vert<\epsilon_w$ we have $F(\boldsymbol{w},\hat{s}(\boldsymbol{w}))\geq F(\boldsymbol{w}^*, \boldsymbol{s}^*)$. By the continuity if $\hat{s}$, $\lim_{\boldsymbol{w}\to\boldsymbol{w}^*}\hat{s}(\boldsymbol{w})=\boldsymbol{s}^*$. Therefore, for small enough $\Vert \boldsymbol{w}-\boldsymbol{w}^*\Vert$, we have $\Vert \hat{s}(\boldsymbol{w})-\boldsymbol{s}^*\Vert<\epsilon_s$. Thus, there exists a neighborhood of $\boldsymbol{w}^*$, $0<\delta\leq\min(\epsilon,\epsilon_w)$ such that $\forall \boldsymbol{w}$ with $\Vert\boldsymbol{w}-\boldsymbol{w}^*\Vert<\delta$ we have $\Vert\hat{s}(\boldsymbol{w})-\boldsymbol{s}^*\Vert<\epsilon_s$ and therefore $F(\boldsymbol{w},\hat{s}(\boldsymbol{w}))\geq F(\boldsymbol{w}^*, \boldsymbol{s}^*)$.
    
    \item[(ii)] Assume next that $\boldsymbol{w}^*$ is a local minimum of $F(\boldsymbol{w},\hat{s}(\boldsymbol{w}))$. Meaning that $\exists \epsilon_W>0$ such that $\forall \boldsymbol{w}$ with $\Vert\boldsymbol{w}-\boldsymbol{w}^*\Vert<\epsilon_w$ we have $F(\boldsymbol{w},\hat{s}(\boldsymbol{w}))\geq F(\boldsymbol{w}^*, \hat{s}(\boldsymbol{w}^*))$. From the definition of $\hat{s}(\boldsymbol{w})$, we know that $F(\boldsymbol{w},\boldsymbol{s})\geq F(\boldsymbol{w},\hat{s}(\boldsymbol{w}))$. Therefore, $\forall \boldsymbol{w}$ with $\Vert\boldsymbol{w}-\boldsymbol{w}^*\Vert<\epsilon_w$ we have $F(\boldsymbol{w},\boldsymbol{s})\geq F(\boldsymbol{w}^*, \hat{s}(\boldsymbol{w}^*))$, making $(\boldsymbol{w}^*, \hat{s}(\boldsymbol{w}^*))$ a local minimum of $F(\boldsymbol{w},\boldsymbol{s})$.
    
    \item[(iii)] To show the value of the minima coincide, we use the definition of $\hat{s}(\boldsymbol{w})$ to write
        $$
            \min_{\boldsymbol{w}}F(\boldsymbol{w},\hat{s}(\boldsymbol{w}))=\min_{\boldsymbol{w}}\left[\min_{\boldsymbol{s}}F(\boldsymbol{w},\boldsymbol{s})\right]=\min_{\boldsymbol{w},\boldsymbol{s}}F(\boldsymbol{w},\boldsymbol{s}).
        $$
        Next, assume that $\boldsymbol{w}^*$ is the global minimum of $F(\boldsymbol{w},\hat{s}(\boldsymbol{w}))$. Then we $F(\boldsymbol{w},\boldsymbol{s})\geq F(\boldsymbol{w},\hat{s}(\boldsymbol{w}))\geq F(\boldsymbol{w}^*, \hat{s}(\boldsymbol{w}^*))$, meaning, $(\boldsymbol{w}^*, \hat{s}(\boldsymbol{w}^*))$ is a global minimum of $F(\boldsymbol{w},\boldsymbol{s})$. The other direction is due to the following. Since $F(\boldsymbol{w},\boldsymbol{s})\geq F(\boldsymbol{w},\hat{s}(\boldsymbol{w}))$ the minimum of $F(\boldsymbol{w},\boldsymbol{s})$ always satisfies $s=\hat{s}(\boldsymbol{w})$. Therefore, the global minimum of $F(\boldsymbol{w},\boldsymbol{s})$ is the same as the global minimum of $F(\boldsymbol{w},\hat{s}(\boldsymbol{w}))$.

\end{proof}

\section{Proximal Operators}
\label{app:proximal}
The Proximal Operator of a function $f(w)$ is the functional 
\begin{align}
    {\rm prox}_f(w)=\underset{u}{\operatorname{argmin}}\left[f(u)+\frac{1}{2}(u-w)^2\right].
\end{align}
Meaning, for any function $f$, it returns the $u$ that minimizes $f(u)+(u-w)^2/2$. Generalization from 1D to any space are trivial. 

In the case of regularized loss, proximal operators become handy once we observe the following. Say that the loss is decomposed into an unregularized loss ${\cal L}_0$ and a regularizer $R$, namely ${\cal L}={\cal L}_0+R$. For any $\alpha>0$, $w^*$ is an extremum of $\ell$ if and only if 
\begin{align}
    w^*={\rm prox}_{\alpha R}(w^*-\alpha\nabla{\cal L}_0(w^*))=\underset{u}{\operatorname{argmin}}\left[\alpha R(u)+\frac{1}{2}(u+\alpha\nabla{\cal L}_0(w^*)-w^*)^2\right].
\end{align}
Based on the definition of the proximal operator,  it is a manner of simple algebra to show the above expression is equivalent to $\nabla {\cal L}(w^*)=0$. In case both ${\cal L}_0$ and $R$ are convex, $w^*$ is therefore the global minimum of  ${\cal L}$.

In the context of learning, one can iteratively obtain $w^*$ through the sequence
\begin{align}
    w^{(t+1)}={\rm prox}_{\alpha R}\left[w^{(t)}-\alpha\nabla{\cal L}_0\left(w^{(t)}\right)\right].
\end{align}
In this context, $\alpha$ is identified with the learning rate. More generality, if we use some optimization algorithm to update our weights (e.g. Adam), such that $w\leftarrow w-\alpha\cdot \delta w$. To incorporate a proximal operator of $R$ as regularization, we will simply update $w$ as  $w\leftarrow{\rm prox}_{\alpha R}(w-\alpha\cdot\delta w)$.

For the case of $L_2$ regularization, the proximal operator is given in closed form by
\begin{align}
    \label{eq:proximal_L2}
    {\rm prox}_{\alpha\lambda_2\vert \cdot\vert^2/2}(w)=\frac{w}{1+\alpha\lambda_2},
\end{align}
this is the result we have used in deriving~\cref{eq:proximal}. For completeness, the proximal operator for the $L_1$ norm is known as the soft-thresholding operator, and is given by
\begin{align}
    {\rm prox}_{\alpha\lambda_1\vert \cdot\vert}(w)=\operatorname{sign}(w)\max\left\{\vert w\vert-\alpha\lambda_1,0\right\}.
\end{align}

\section{Non-increasing Loss under $p$WD}
\label{app:non_increasing}
In this appendix, we show that the original loss function is non-increasing under the $p$WD step. 

\begin{lemma}
    \label{lemma:proximal_step_bound}
    Let ${\cal L}:\;\mathbb{R}^n\to\mathbb{R}$ be a continuous differentiable function,, assume the gradient of ${\cal L}$ is Lipschitz continuous with constant $L$, namely $\Vert\nabla{\cal L}(\boldsymbol{w})-\nabla{\cal L}(\boldsymbol{v})\Vert\leq L\Vert\boldsymbol{w}-\boldsymbol{v}\Vert$. Let $R:\;\mathbb{R}^n\to\mathbb{R}$ be a convex function. Define $F(\boldsymbol{w})={\cal L}(\boldsymbol{w})+R(\boldsymbol{w})$.
    
    The sequence $\boldsymbol{w}^{(t+1)}={\rm prox}_{\alpha R}(\boldsymbol{w}^{(t)}-\alpha\nabla{\cal L}(\boldsymbol{w}^{(t)}))$ satisfies

    \item[(i)]
    \begin{align}
        \label{eq:proximal_step_bound}
        R(\boldsymbol{w}^{(t+1)})\leq R(\boldsymbol{w}^{(t)})-\frac{1}{2\alpha}\Vert \boldsymbol{w}^{(t+1)}-\boldsymbol{w}^{(t)}\Vert^2-\left\langle\nabla{\cal L}(\boldsymbol{w}^{(t)}),\boldsymbol{w}^{(t+1)}-\boldsymbol{w}^{(t)}\right\rangle.
    \end{align}
    \item[(ii)]
    \begin{align}
        \label{eq:proximal_step_bound_2}
        F(\boldsymbol{w}^{(t+1)})\leq F(\boldsymbol{w}^{(t)})-\frac{1-\alpha L}{2\alpha}\Vert \boldsymbol{w}^{(t+1)}-\boldsymbol{w}^{(t)}\Vert^2.
    \end{align}
\end{lemma}
\begin{proof}
    See~\citep[Thm. 11.3]{garrigos2024handbook}
\end{proof}

We prove the following theorem.
\begin{theorem}
    \label{thm:non_increasing}
    Let ${\cal L}:\;\mathbb{R}^n\to\mathbb{R}$ be a continuous differentiable function, assume the gradient of ${\cal L}$ is Lipschitz continuous with constant $L$, namely $\Vert\nabla{\cal L}(\boldsymbol{w})-\nabla{\cal L}(\boldsymbol{v})\Vert\leq L\Vert\boldsymbol{w}-\boldsymbol{v}\Vert$. Define $F(\boldsymbol{w})={\cal L}(\boldsymbol{w})+(\lambda_p/p)\Lp{p}{\boldsymbol{w}}$. Then, the sequence
    $$
        \boldsymbol{w}^{(t+1)}=\frac{\boldsymbol{w}^{(t)}-\alpha\nabla{\cal L}(\boldsymbol{w}^{(t)})}{1+\alpha\lambda_p\vert\boldsymbol{w}^{(t)}\vert^{p-2}}
    $$
    is such that 
    $$
        F(\boldsymbol{w}^{(t+1)})\leq F(\boldsymbol{w}^{(t)})-\frac{1-\alpha_tL}{2\alpha_t}\Vert \boldsymbol{w}^{(t+1)}-\boldsymbol{w}^{(t)}\Vert^2. 
    $$
    Clearly, it is non-increasing, provided $\alpha_t\leq 1/L$.
\end{theorem}

\begin{proof}
    We will use the generalized loss function $F(\boldsymbol{w},\boldsymbol{s})={\cal L}(\boldsymbol{w})+R_p(\boldsymbol{w},\boldsymbol{s})$, where $R_p(\boldsymbol{w},\boldsymbol{s})$ is defined in~\cref{eq:reg_dual}. Importantly, $F(\boldsymbol{w})=\min_{\boldsymbol{s}}F(\boldsymbol{w},\boldsymbol{s})$. We assume the following update rule:
    $$
        \boldsymbol{w}^{(t+1)}={\rm prox}_{\alpha_t R_p(\cdot,\boldsymbol{s}_t)}\left(\boldsymbol{w}^{(t)}-\alpha_t\nabla{\cal L}(\boldsymbol{w}^{(t)})\right)=\frac{\boldsymbol{w}^{(t)}-\alpha_t\nabla{\cal L}(\boldsymbol{w}^{(t)})}{1+\lambda_p \alpha_t\boldsymbol{s}^{(t)}}\;\;,\;\;\boldsymbol{s}^{(t+1)}=\vert\boldsymbol{w}^{(t+1)}\vert^{p-2},
    $$
    where all the operations are done element-wise. We note that given the update rule of $\boldsymbol{s}$ above, the update rule of $\boldsymbol{w}$ is the same as the one in the theorem. At fixed $\boldsymbol{s}=\boldsymbol{s}^{(t)}$, we can apply~\cref{lemma:proximal_step_bound} to $F(\boldsymbol{w},\boldsymbol{s}^{(t)})$, and obtain
    $$
        F(\boldsymbol{w}^{(t+1)},\boldsymbol{s}^{(t)})\leq F(\boldsymbol{w}^{(t)},\boldsymbol{s}^{(t)})-\frac{1-\alpha_tL}{2\alpha_t}\Vert \boldsymbol{w}^{(t+1)}-\boldsymbol{w}^{(t)}\Vert^2.
    $$
    From~\cref{app:equivalence,eq:s_opt_der}, we know that the update rule for $\boldsymbol{s}$ is such that $F(\boldsymbol{w}^{(t)})=F(\boldsymbol{w}^{(t)},\boldsymbol{s}^{(t)})\leq F(\boldsymbol{w}^{(t)},\boldsymbol{s})$ for any $\boldsymbol{s}$. Therefore, we have
    $$
        F(\boldsymbol{w}^{(t+1)})\leq F(\boldsymbol{w}^{(t)})-\frac{1-\alpha_tL}{2\alpha_t}\Vert \boldsymbol{w}^{(t+1)}-\boldsymbol{w}^{(t)}\Vert^2.
    $$
\end{proof}
We note that while very similar steps are the base for the proof of general convergence of proximal gradient methods~\citep[E.g. Thm. 11.3]{garrigos2024handbook}, our case is more involved due to the $\boldsymbol{s}$ dependence of the proximal operator.

\section{Stability of the $w_i=0$ Fixed Point}
\label{app:stability}
In this appendix, we discuss the stability of the $w_i=0$ fixed point of the proximal gradient step. We will focus on a single weight case, the generalization to a higher dimension follows trivially. Starting from the proximal gradient step in~\cref{eq:proximal}, we want to ask whether a point arbitrarily close to $w=0$ will be driven to $w=0$ by the proximal gradient step. We will show that this is indeed the case for $p<1$, while for $p>1$ the fixed point is unstable. We assume that $w=0$ is not the global minimum of the unregularized loss, and assume that the weight prior to the proximal gradient step is small compared to the weight update by the unregularized loss, namely, 
\begin{align}
    \label{eq:small_w}
    w=\epsilon\;\;:\;\;\vert\epsilon\vert\ll\alpha\vert\delta w\vert.
\end{align}
This means that the current proximal gradient step is given by
\begin{align}
    \label{eq:small_w_step}
    w\gets\frac{\epsilon-\alpha\delta w}{1+\alpha\lambda_p\epsilon^{p-2}}\simeq -\frac{\alpha \epsilon^{2-p}\delta w}{\epsilon^{2-p}+\alpha\lambda_p}\simeq-\frac{\delta w \epsilon^{2-p}}{\lambda_p}.
\end{align}
In the last step we have assumed that $\alpha\lambda_p\epsilon^{p-2}\gg1$, which is a necessary condition for the proximal gradient step to drive the weight to zero. The ratio of the updated weight to the original weight is thus given by
\begin{align}
    \label{eq:ratio}
    \left\vert\frac{w}{\epsilon}\right\vert\simeq\frac{\vert\delta w\vert}{\lambda_p}\vert\epsilon\vert^{1-p}.
\end{align}
For sufficiently small $\vert\epsilon\vert$, the above ratio is smaller than 1 for $p<1$, and larger than 1 for $p>1$. This means that $w=0$ is a stable fixed point for $p<1$, and an unstable for $p>1$. For $p=1$, the fixed point is stable for $\vert\delta w\vert<\lambda_1$, as it should for $L_1$ regularization.

\section{$p$Adam Code}
\label{app:padam_code}

Here, we provide an example for implementing $p$WD on the standard Adam algorithm.

\captionsetup[lstlisting]{labelsep=colon}

\lstdefinestyle{custompython}{
  belowcaptionskip=1\baselineskip,
  breaklines=true,
  frame=L,
  xleftmargin=\parindent,
  language=Python,
  showstringspaces=false,
  basicstyle=\footnotesize\ttfamily,
  keywordstyle=\bfseries\color{green!40!black},
  commentstyle=\itshape\color{purple!40!black},
  identifierstyle=\color{blue},
  stringstyle=\color{orange},
}

\begin{lstlisting}[style=custompython, caption={\texttt{PyTorch} pAdam optimizer implementation}, label=padam]
    class pAdam(torch.optim.AdamW):
    def __init__(self, params, lr=1e-3, betas=(0.9, 0.999), eps=1e-8, lambda_p=1e-2, p_norm=1, *args, **kwargs):
        super(pAdam, self).__init__(params, lr=lr, betas=betas, eps=eps, weight_decay=0, *args, **kwargs)
        self.p_norm = p_norm
        self.lambda_p = lambda_p
    

    @torch.no_grad()
    def step(self, closure=None):
        # Store the old params
        old_params = []
        for group in self.param_groups:
            old_params.append({param: param.data.clone() for param in group['params'] if param.grad is not None})

        # Perform the standard AdamW step
        loss = super(pAdam, self).step(closure)

        # Perform the pWD step
        for group, old_group in zip(self.param_groups, old_params):
            lambda_p_group = group.get('lambda_p', self.lambda_p)  # support prams groups 
            if lambda_p_group > 0:  # Apply regularization only for lambda_p > 0
                for param in group['params']:
                    if param.grad is None:
                        continue

                    # Use old parameters in the decay factor
                    param_old = old_group[param]
                    X = param_old.abs()**(2 - self.p_norm)
                    update_term = X / (X + self.p_norm * group['lr'] * lambda_p_group)

                    # pWD step
                    param.data.mul_(update_term)

        return loss
\end{lstlisting}

\section{Experimental Details}
\label{app:experimental_details}
In this section, we provide additional details on the experimental setup and hyperparameters used in our experiments. We include supplementary figures that were omitted from the main text.

In all experiments we used Adam as our base optimizer. We held the Adam hyperparameters constant for all experiments:
\begin{itemize}
    \item $\beta_1=0.9$.
    \item $\beta_2=0.999$.
    \item $\epsilon=10^{-8}$.
\end{itemize}

We used a learning rate schedule comprised of a linear warm-up, up to \texttt{max\_lr}, followed by a cosine annealing reaching a minimum learning rate of \texttt{min\_lr}=\texttt{max\_lr}/100.

\subsection{ResNet18 on CIFAR-10}
We used the standard ResNet18 architecture for our experiments. We trained the network for 100 epochs with a batch size of 64, and 4 workers for data loading. The linear warm-up was set to 3 epochs. We scanned \texttt{max\_lr} and $\lambda_p$ for a range of $p$ values. The accuracy contours are shown below in~\cref{fig:resnet_contours}.

\begin{figure*}[t]
    \vskip 0.2in
    \begin{center}
    \centerline{\includegraphics[width=0.3\textwidth]{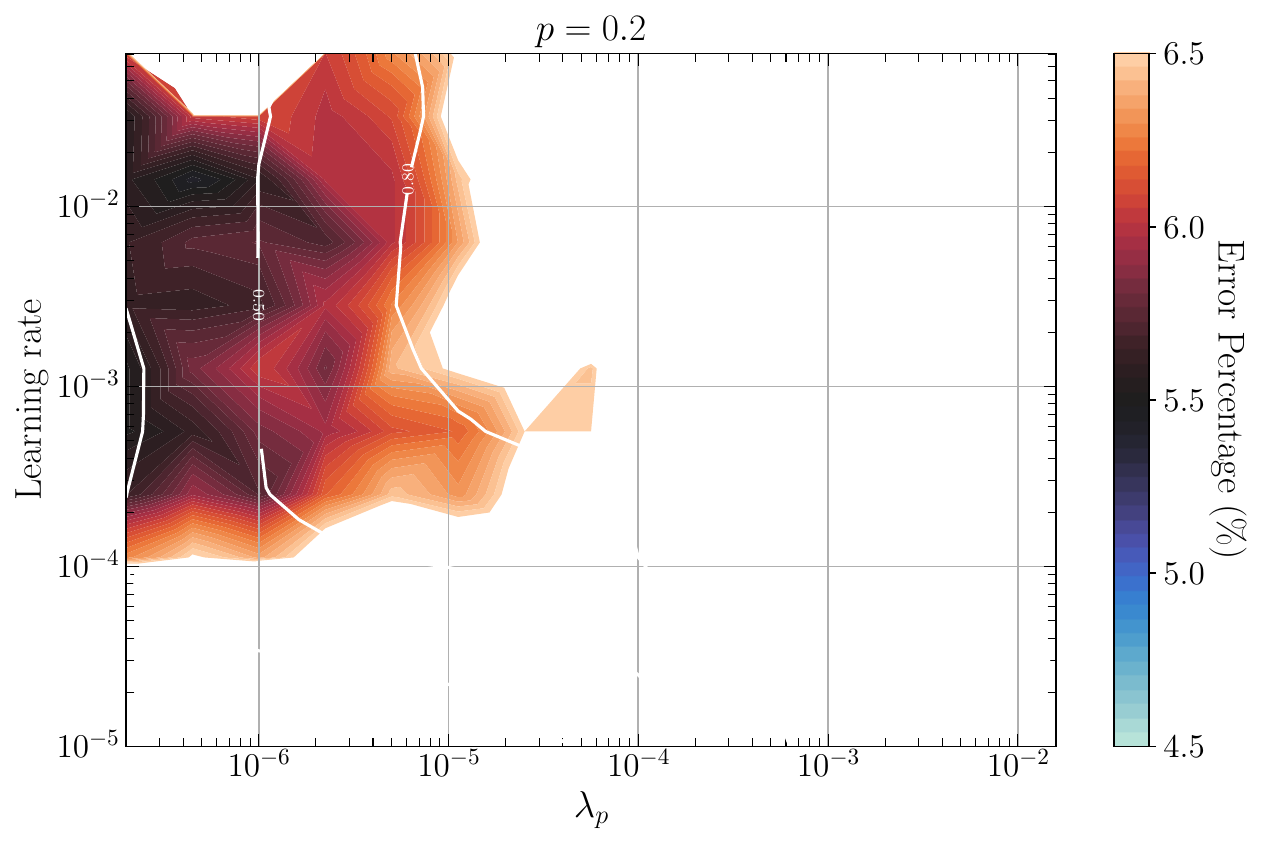}\includegraphics[width=0.3\textwidth]{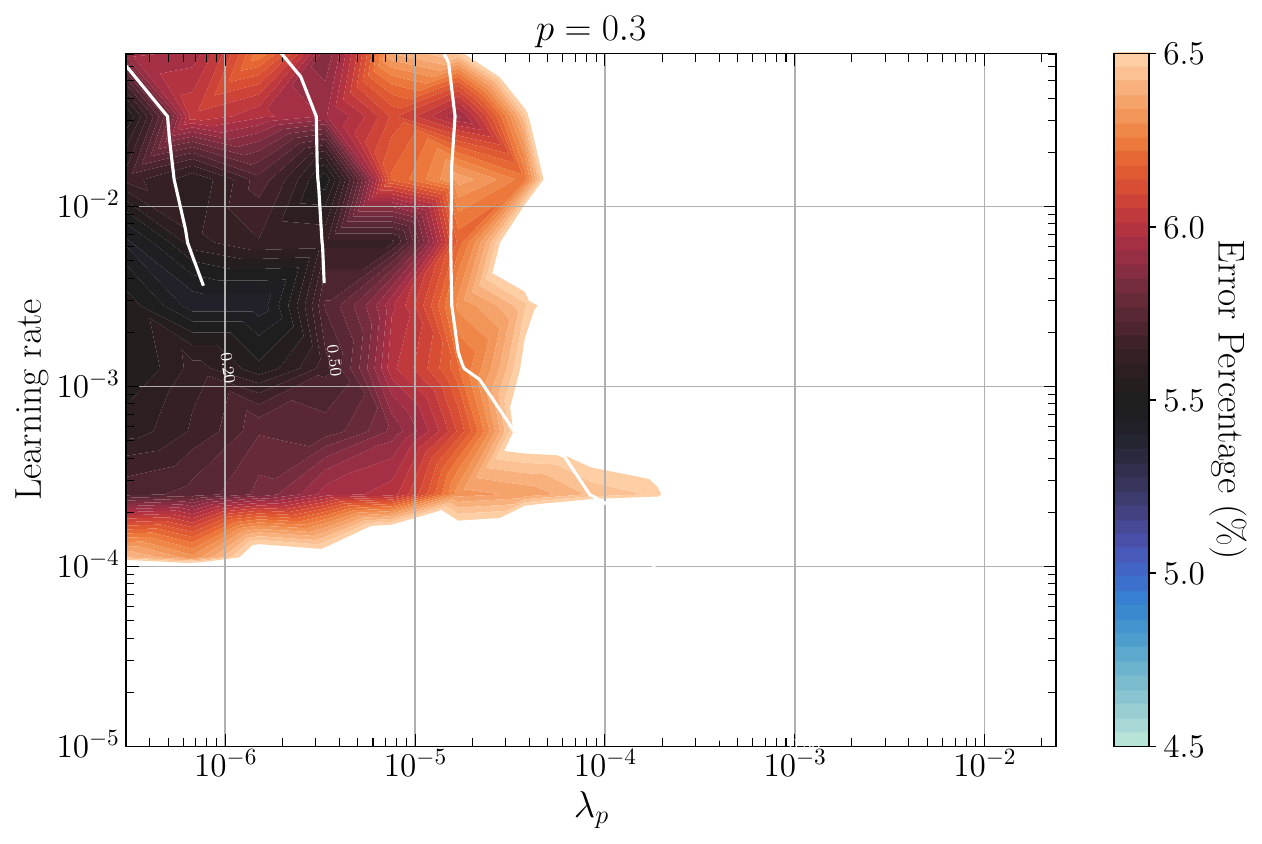}\includegraphics[width=0.3\textwidth]{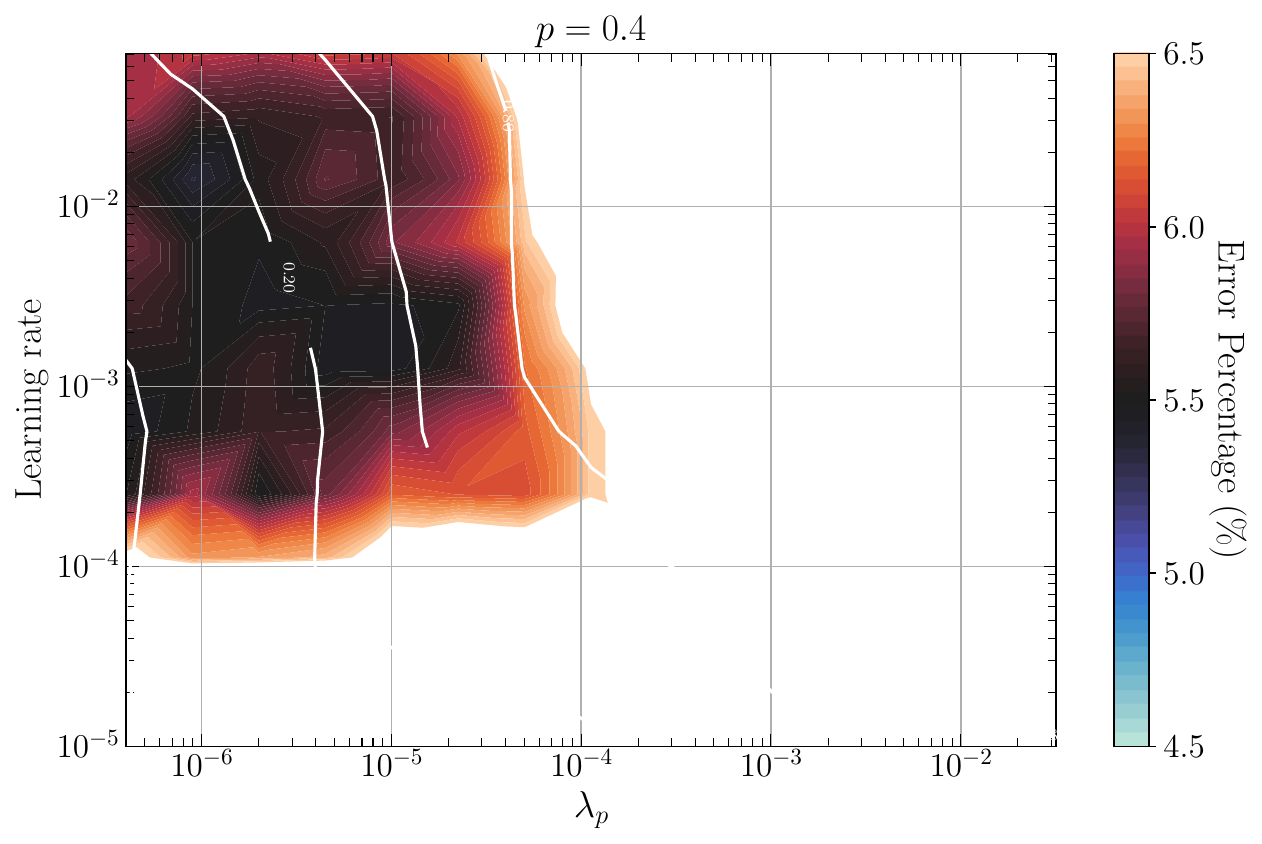}}
    \centerline{\includegraphics[width=0.3\textwidth]{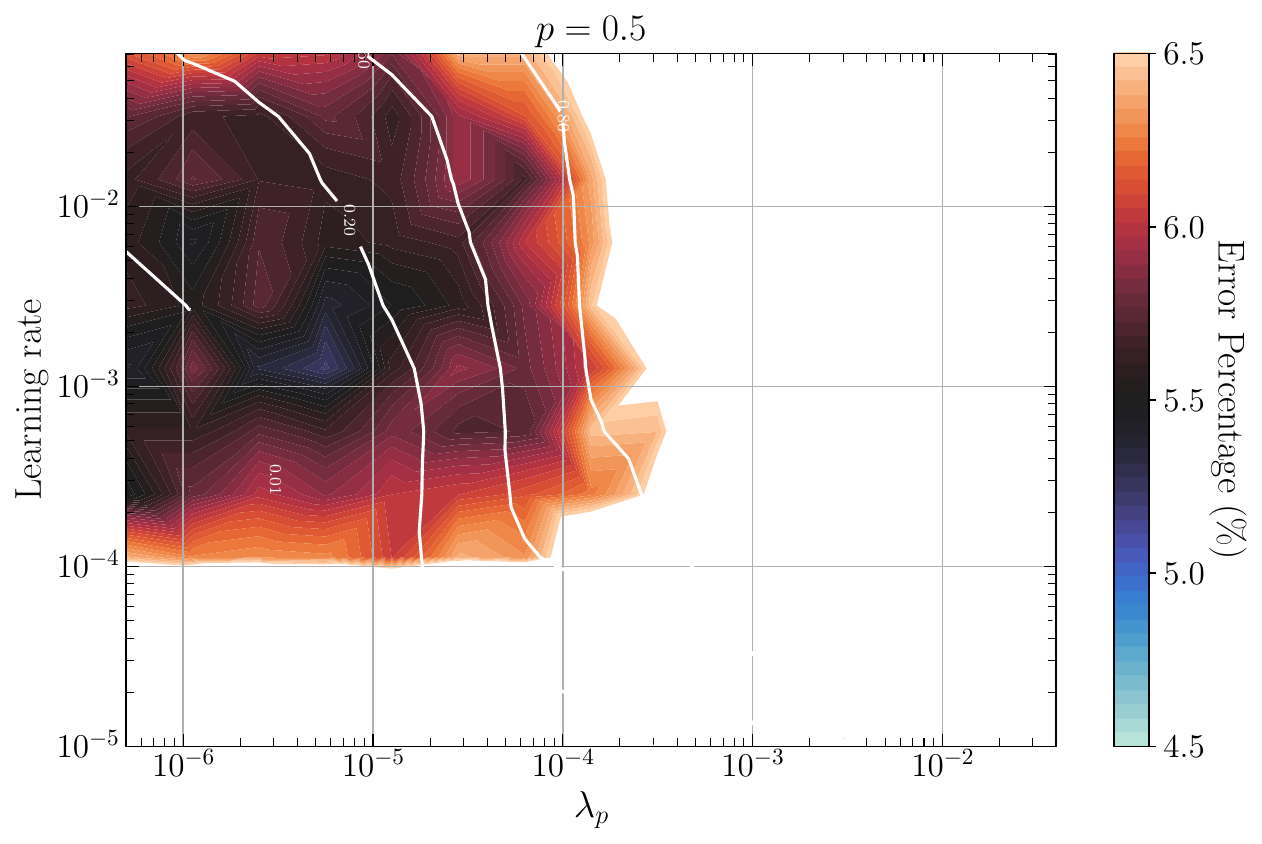}\includegraphics[width=0.3\textwidth]{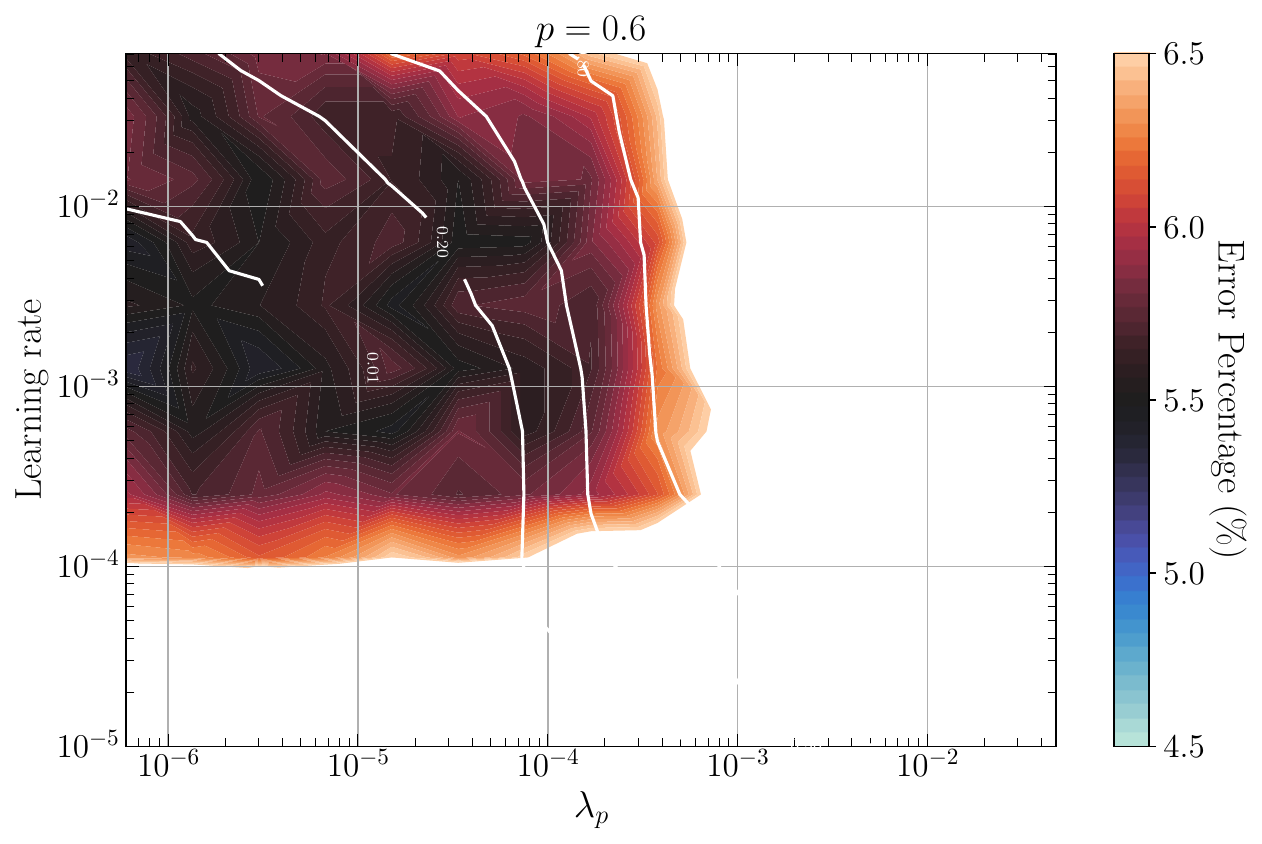}\includegraphics[width=0.3\textwidth]{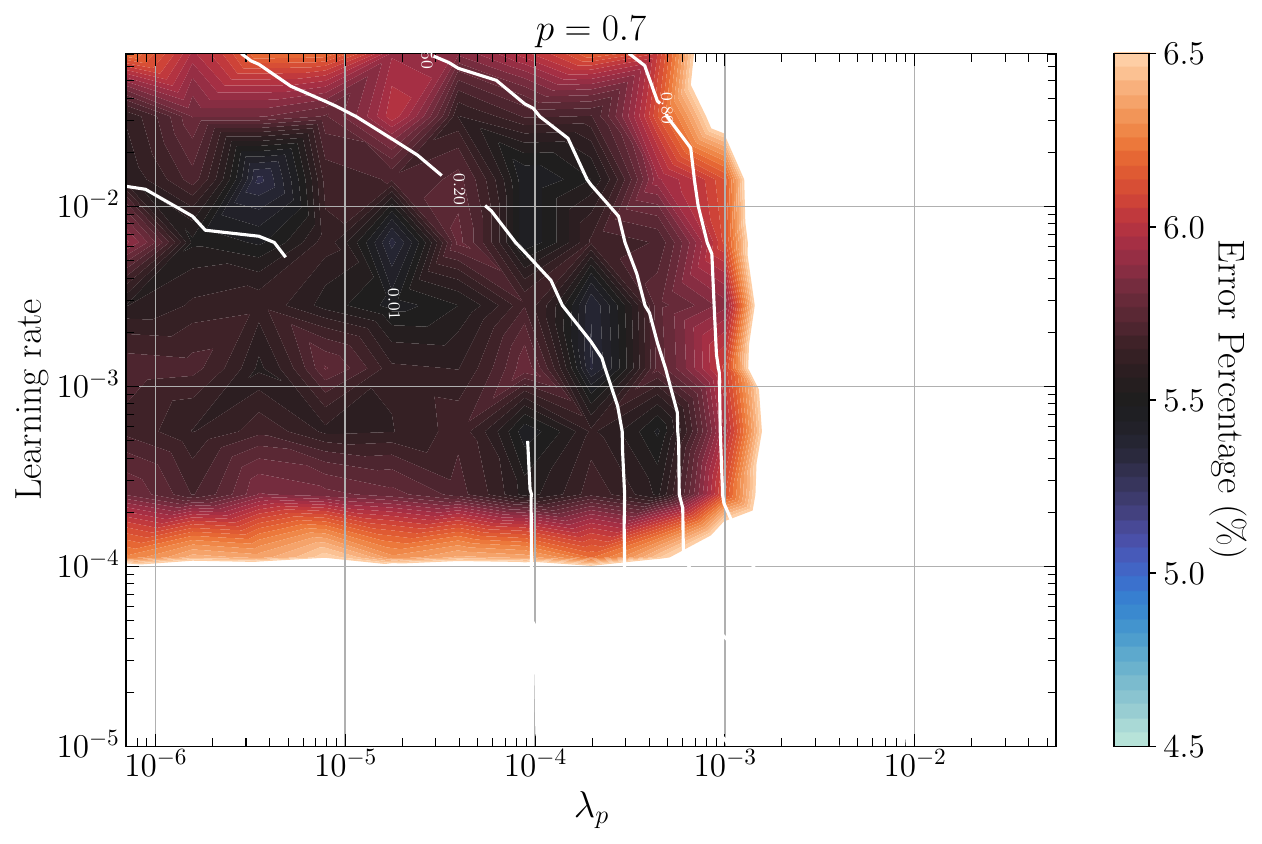}}
    \centerline{\includegraphics[width=0.3\textwidth]{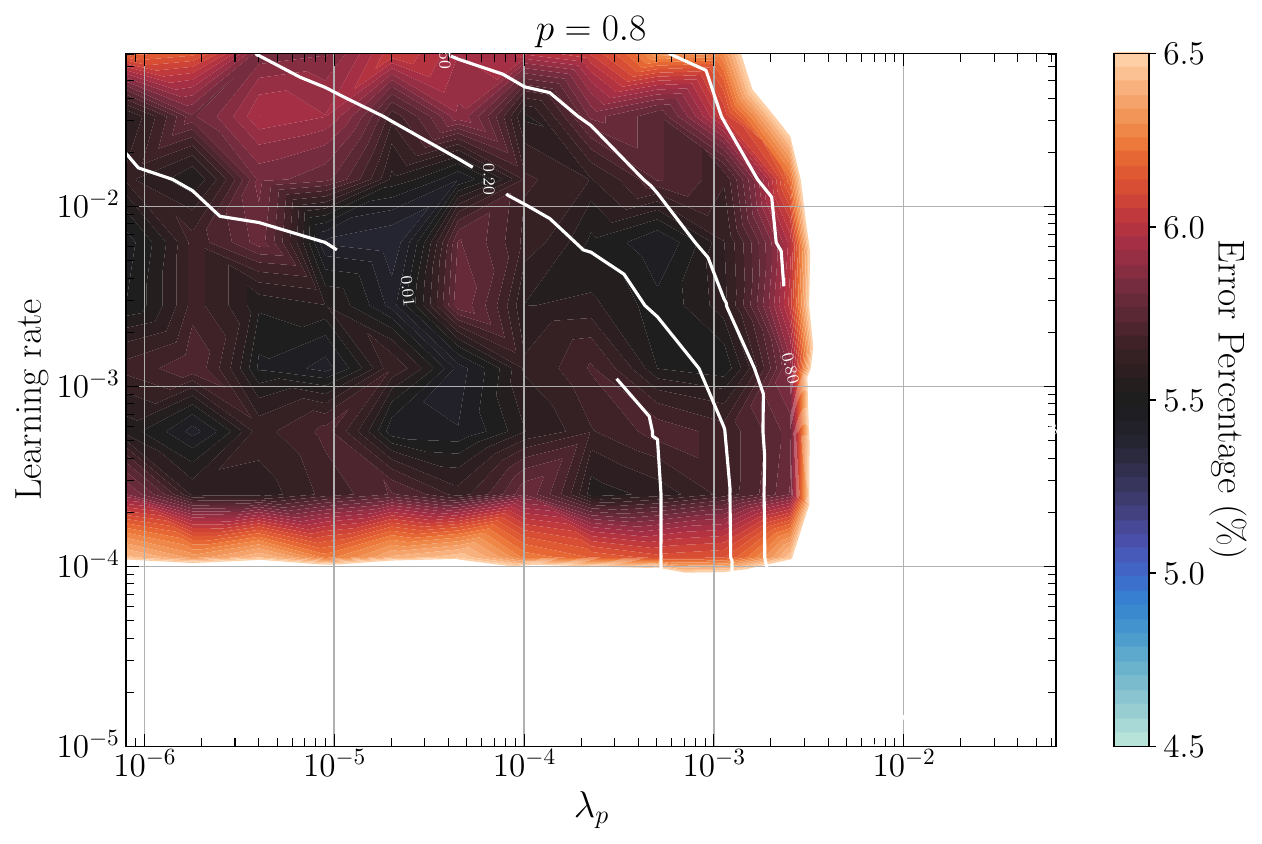}\includegraphics[width=0.3\textwidth]{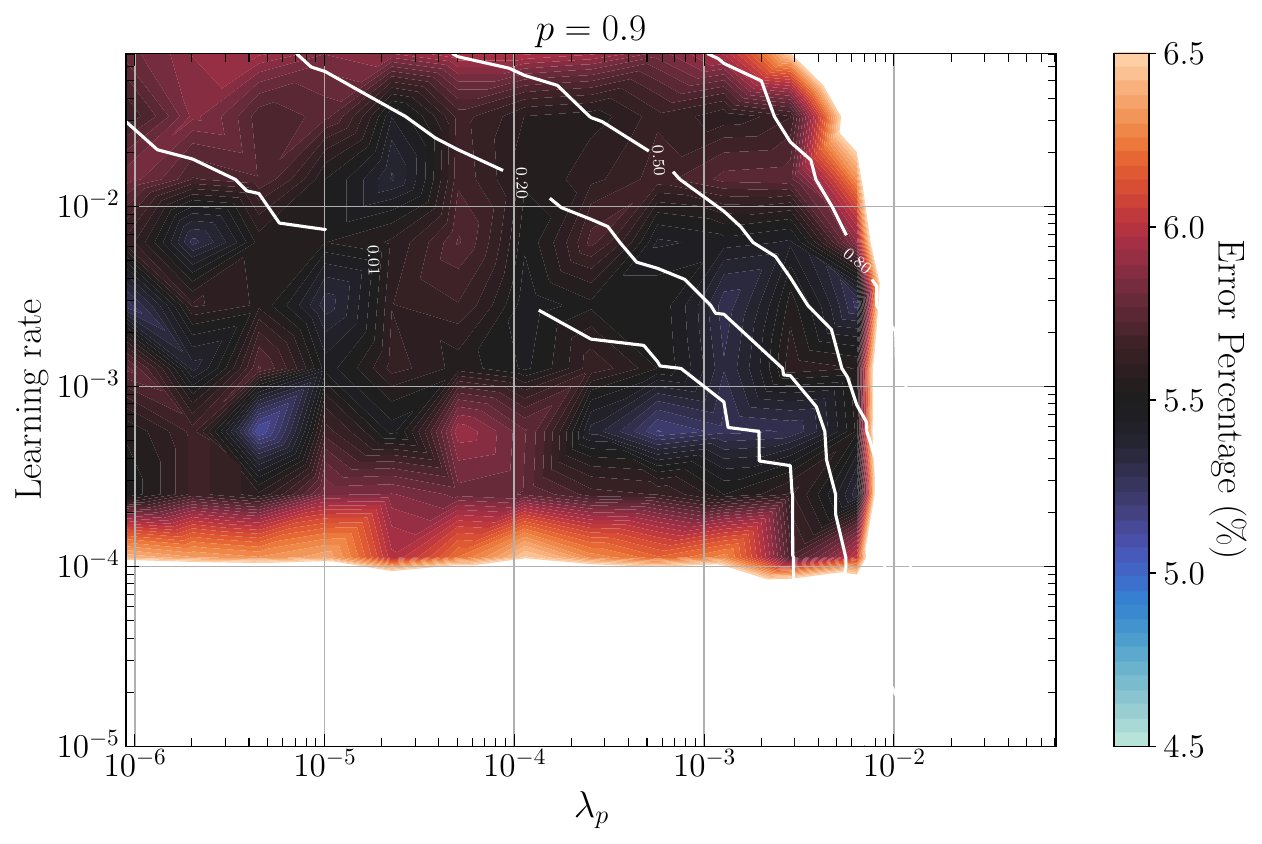}\includegraphics[width=0.3\textwidth]{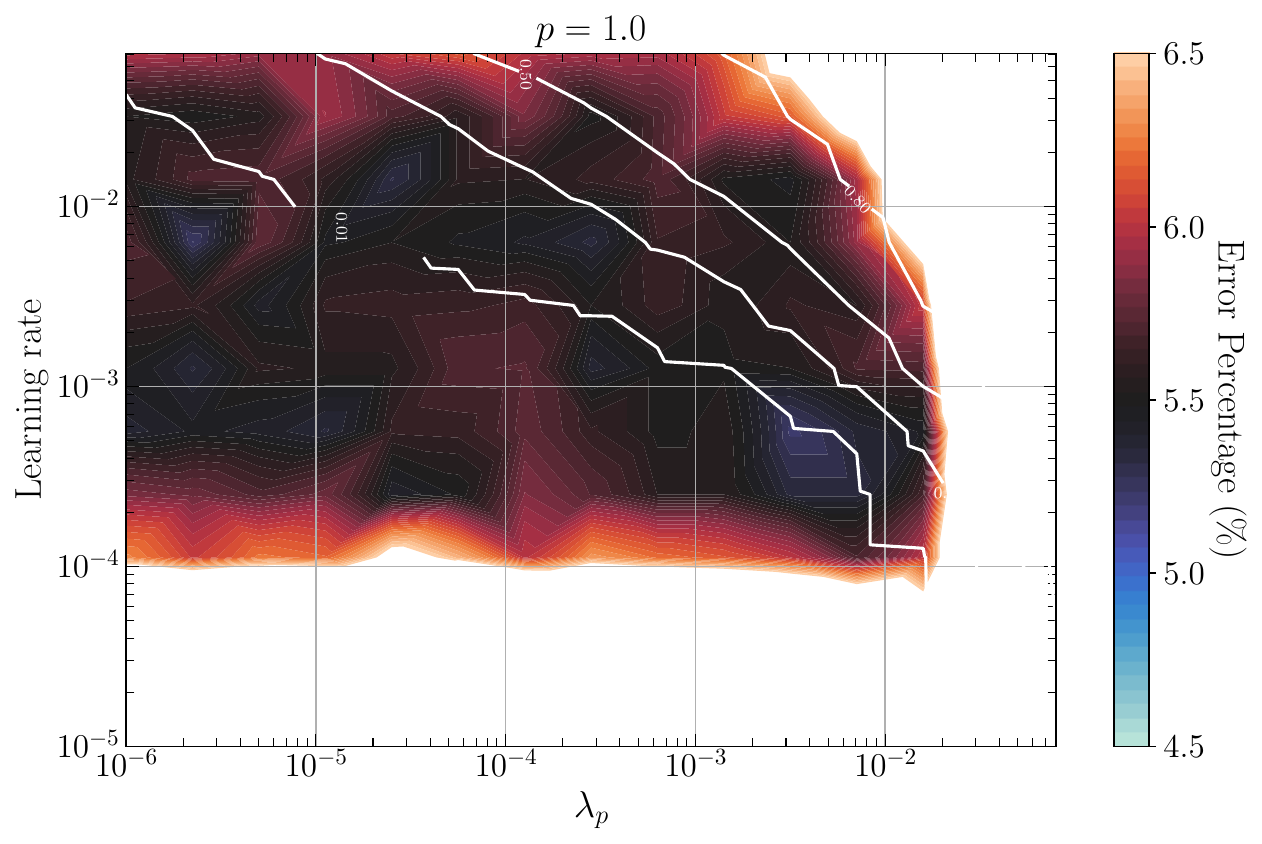}}
    \centerline{\includegraphics[width=0.3\textwidth]{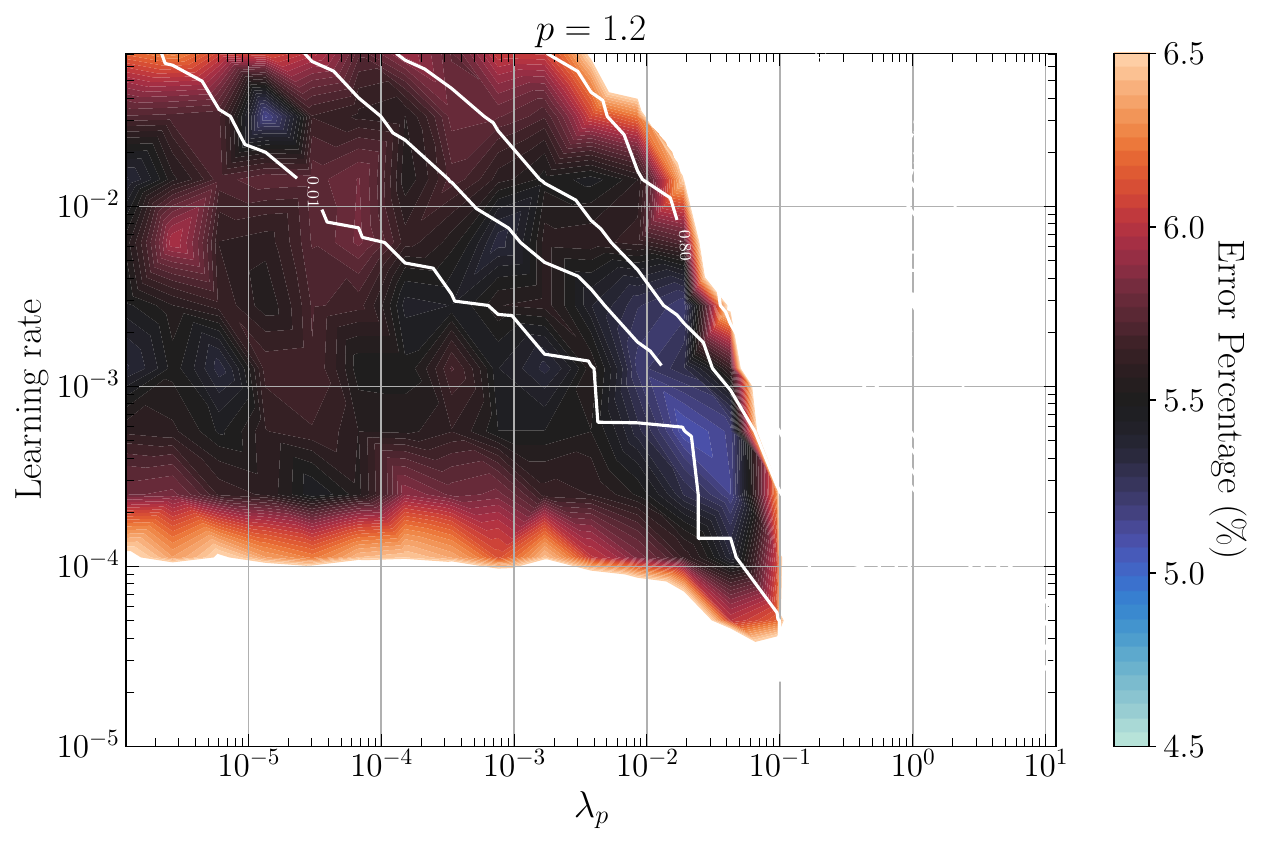}\includegraphics[width=0.3\textwidth]{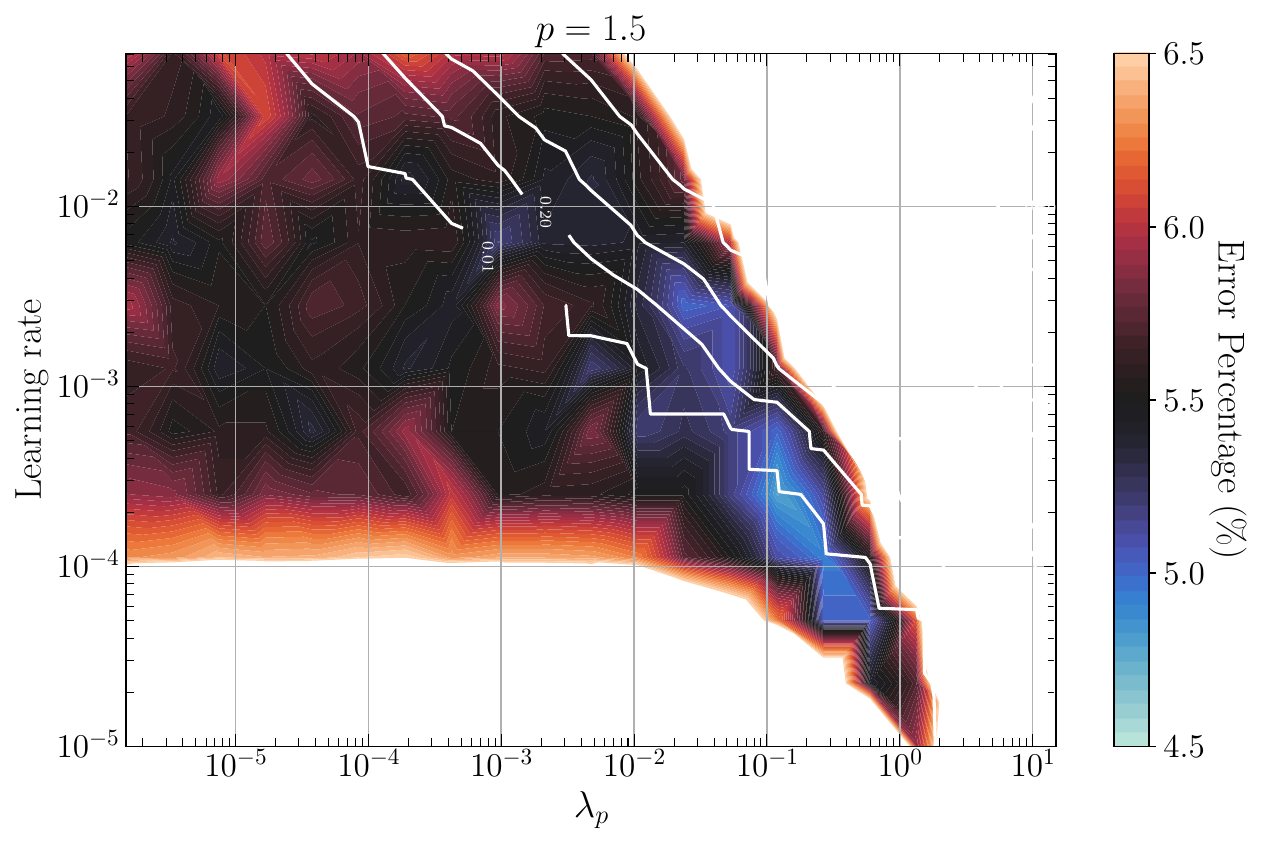}\includegraphics[width=0.3\textwidth]{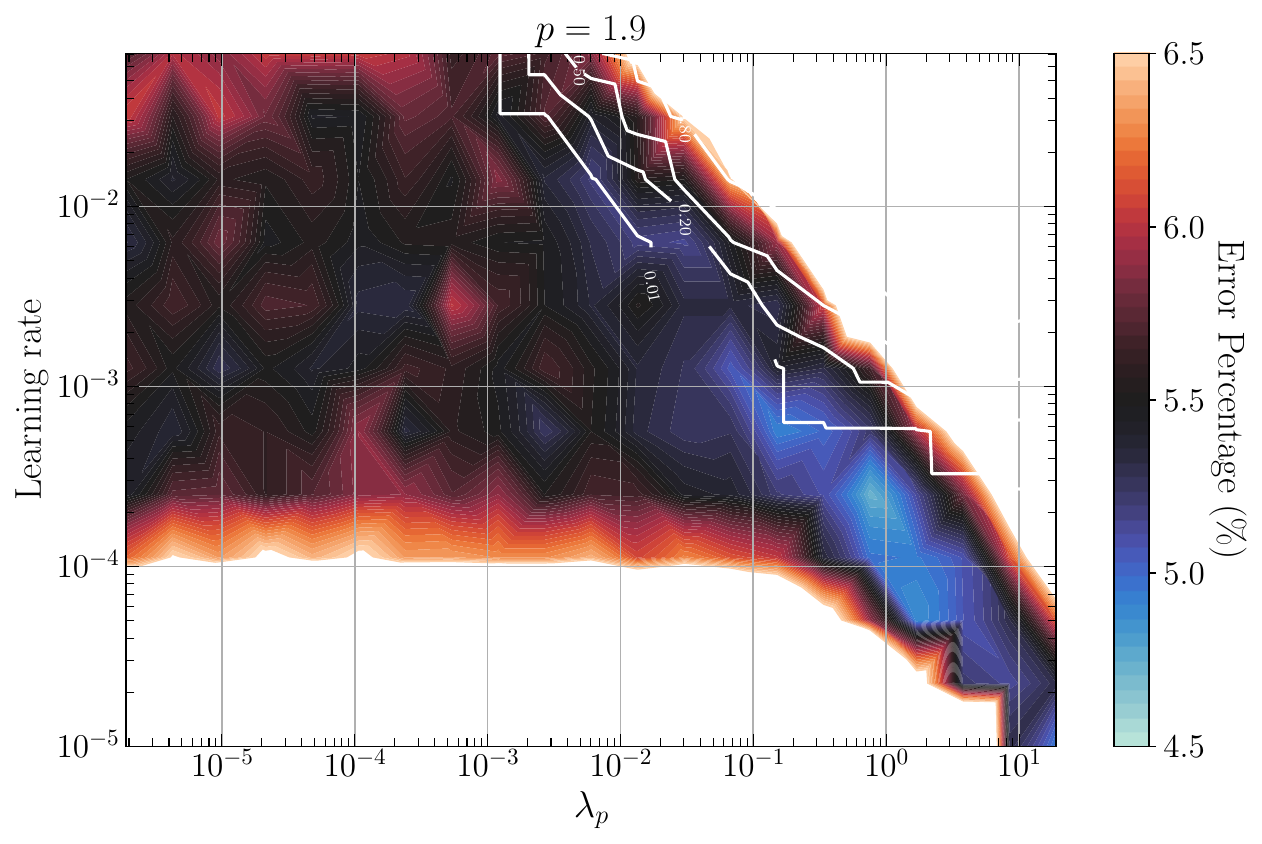}}
    \caption{
        Contours of validation accuracy after 100 training epochs on the $\lambda_p$ vs. learning rate plane, for ResNet18 on CIFAR-10. White contours represent the [0.01, 0.2, 0.4, 0.8] sparisty levlel. 
    }
    \label{fig:resnet_contours}
    \end{center}
    \vskip -0.2in
\end{figure*}

\subsection{nanoGPT on Tiny Shakespeare}
We used the nanoGPT architecture for our experiments. We trained the network for 5000 iterations. We used a batch size of 64, block size of 256, 6 attention heads, 6 layers, embedding dimension of size 384, and gradient clipping of 1.0. We scanned \texttt{max\_lr} and $\lambda_p$ for a range of $p$ values. The linear warm-up was set to 100 iterations. The accuracy contours are shown below in~\cref{fig:nanoGPT_contours}.

\begin{figure*}[t]
    \vskip 0.2in
    \begin{center}
    \centerline{\includegraphics[width=0.3\textwidth]{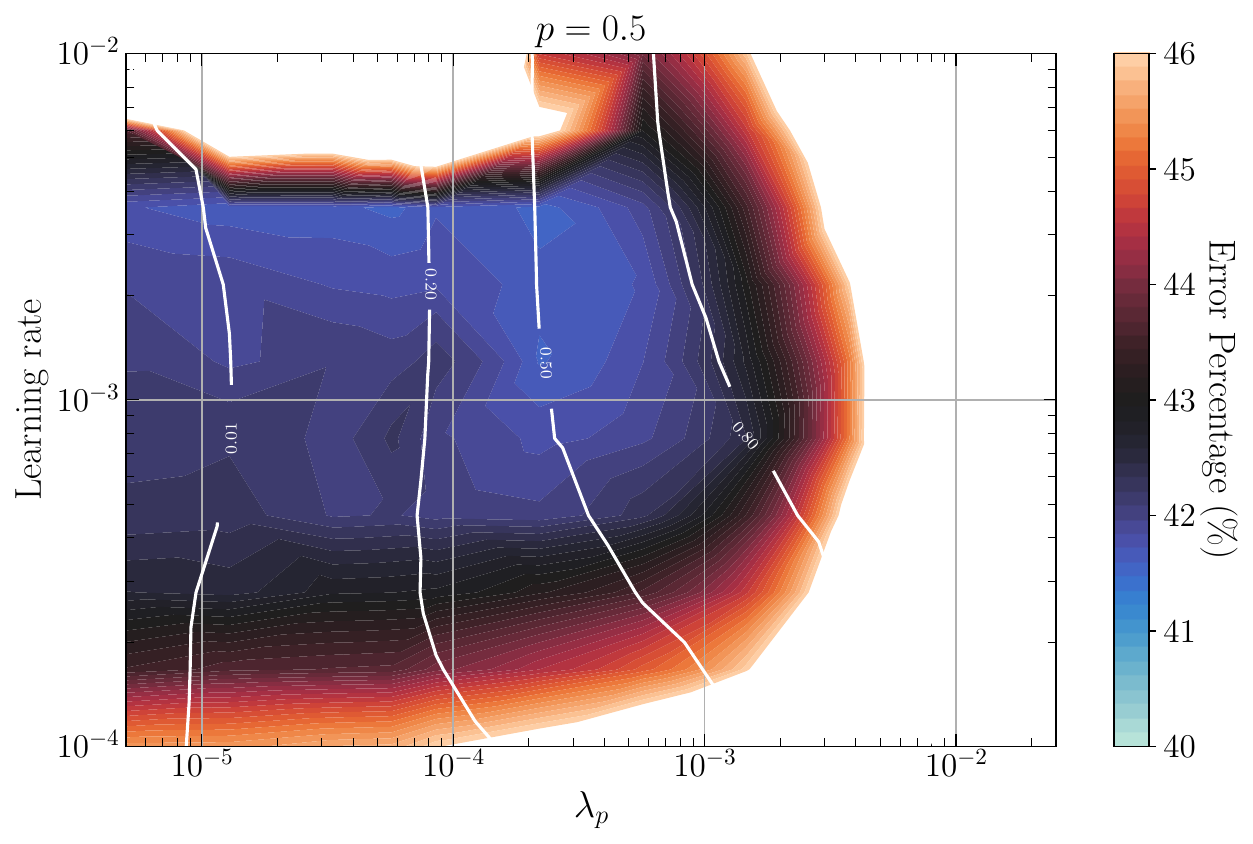}\includegraphics[width=0.3\textwidth]{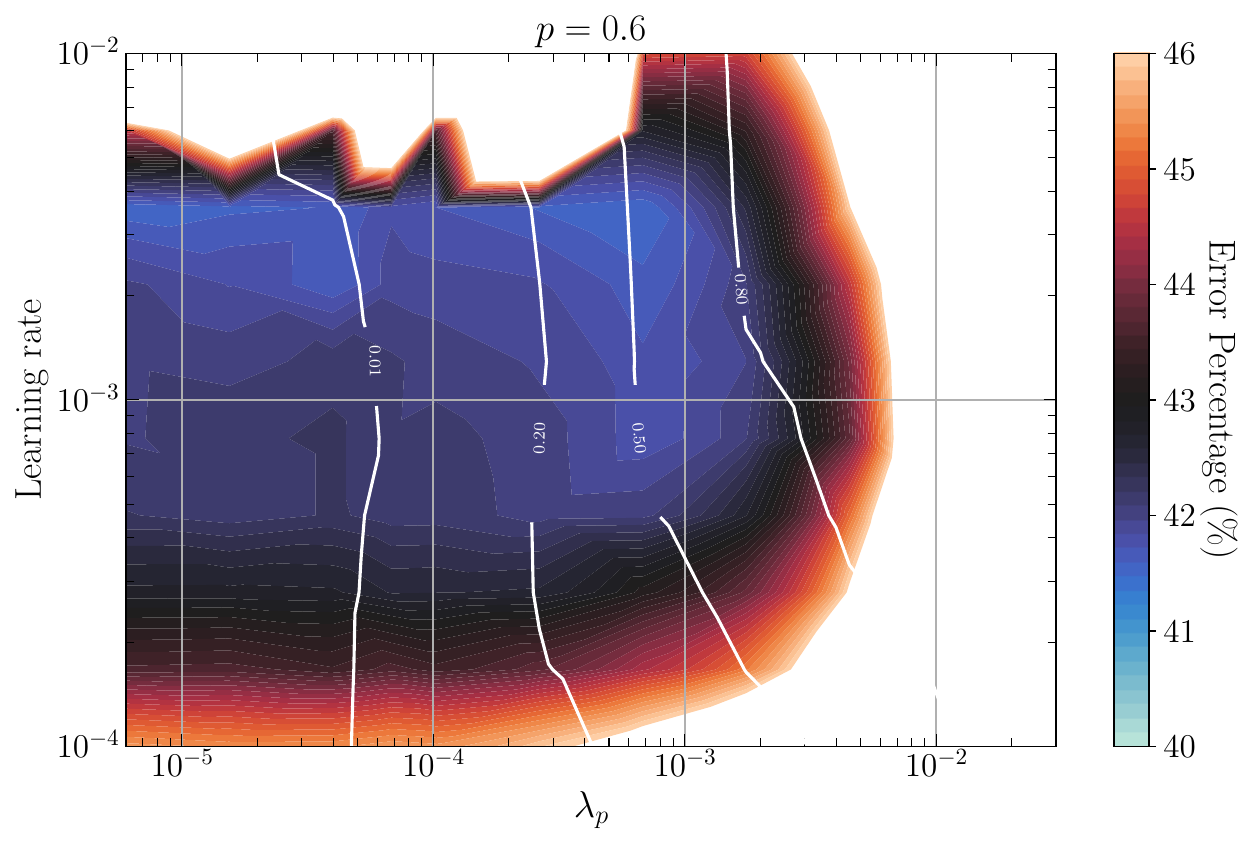}\includegraphics[width=0.3\textwidth]{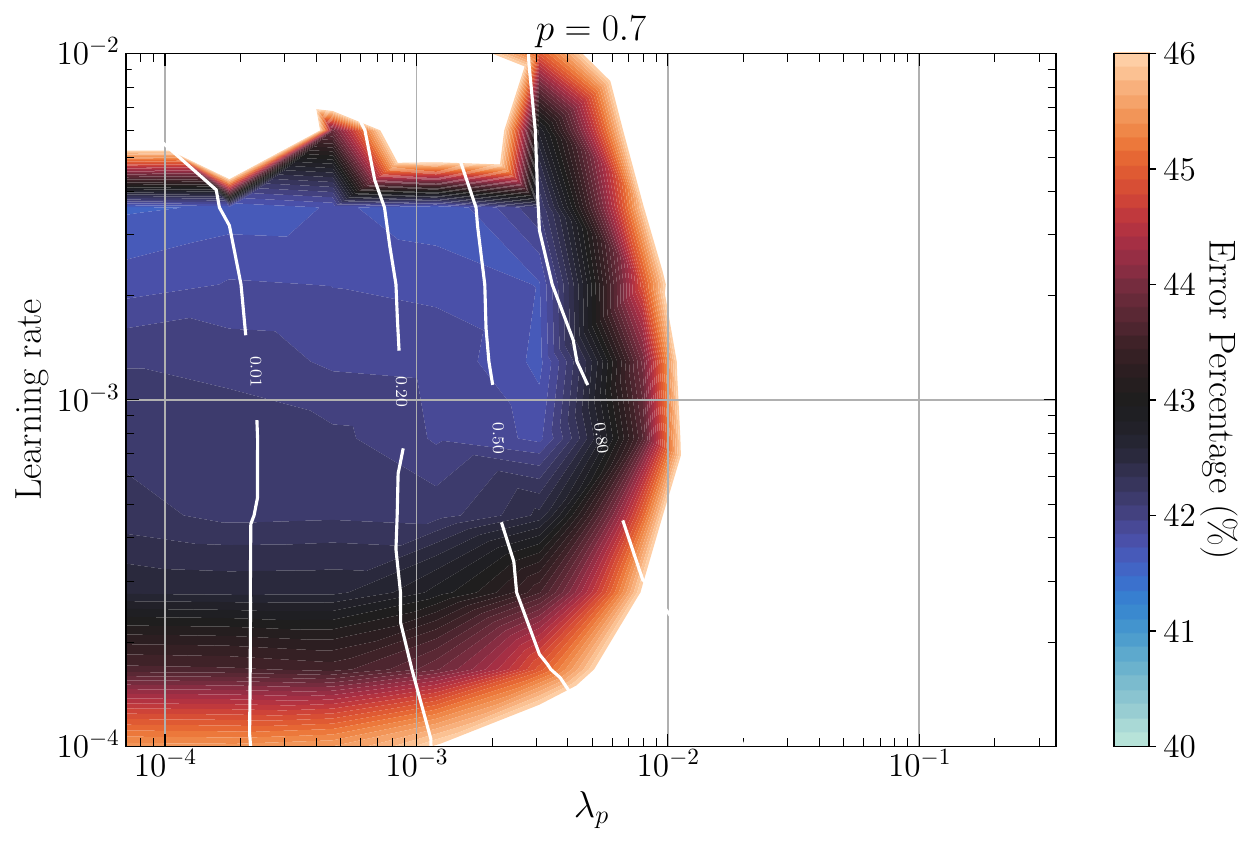}}
    \centerline{\includegraphics[width=0.3\textwidth]{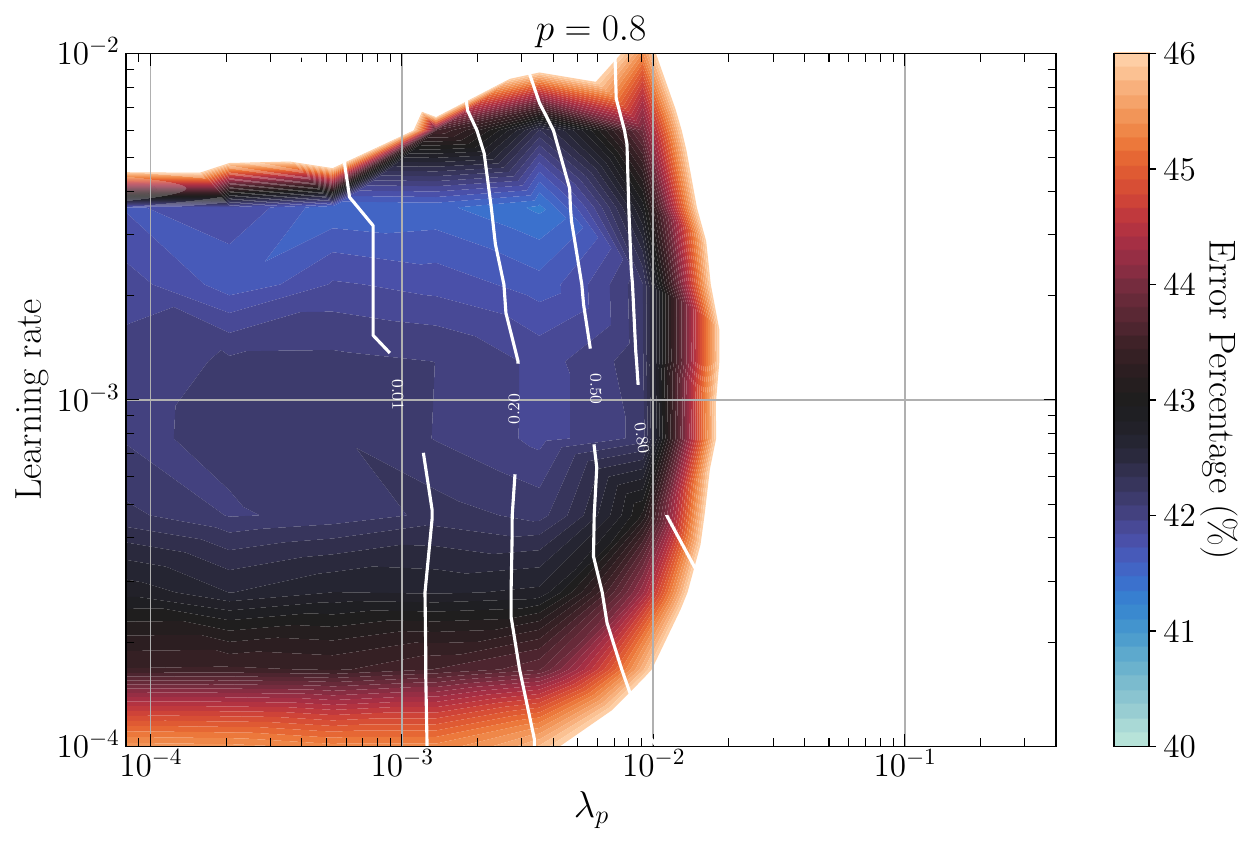}\includegraphics[width=0.3\textwidth]{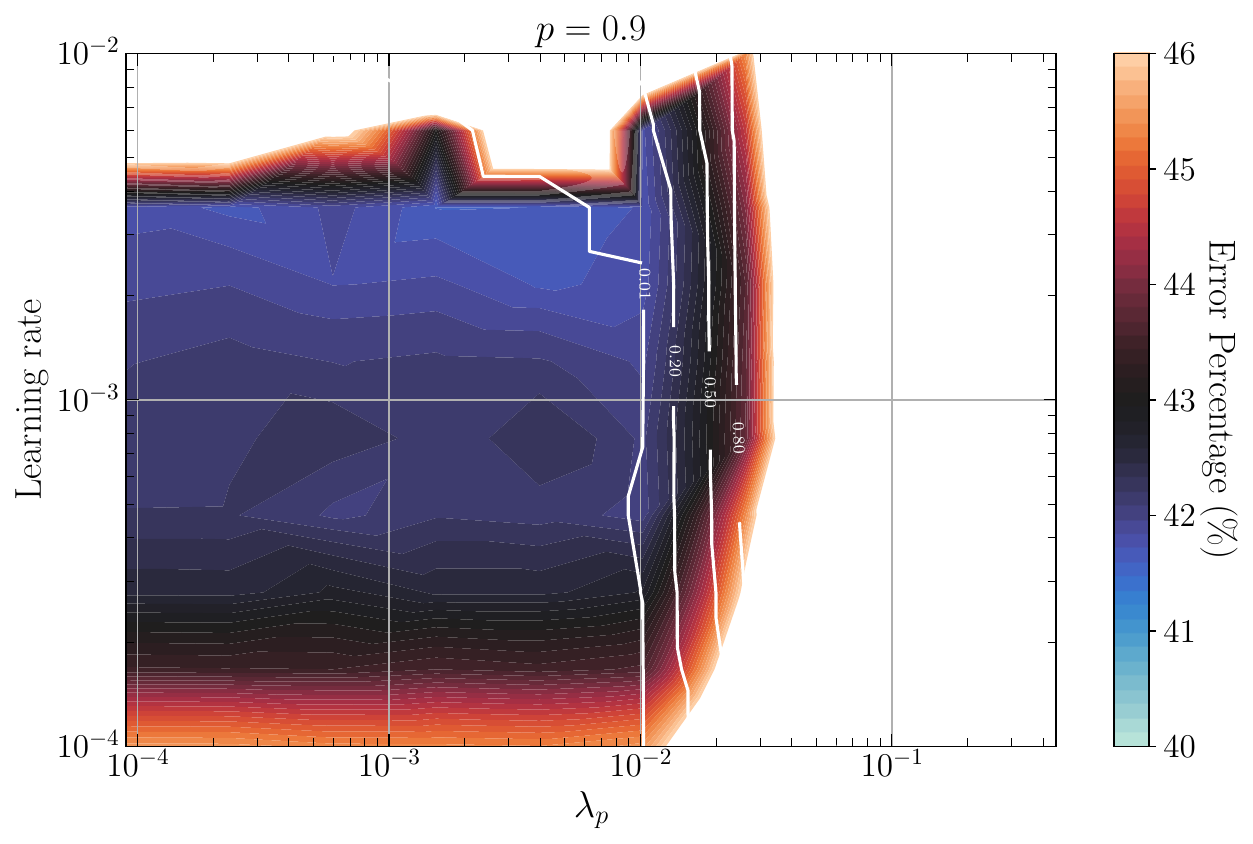}\includegraphics[width=0.3\textwidth]{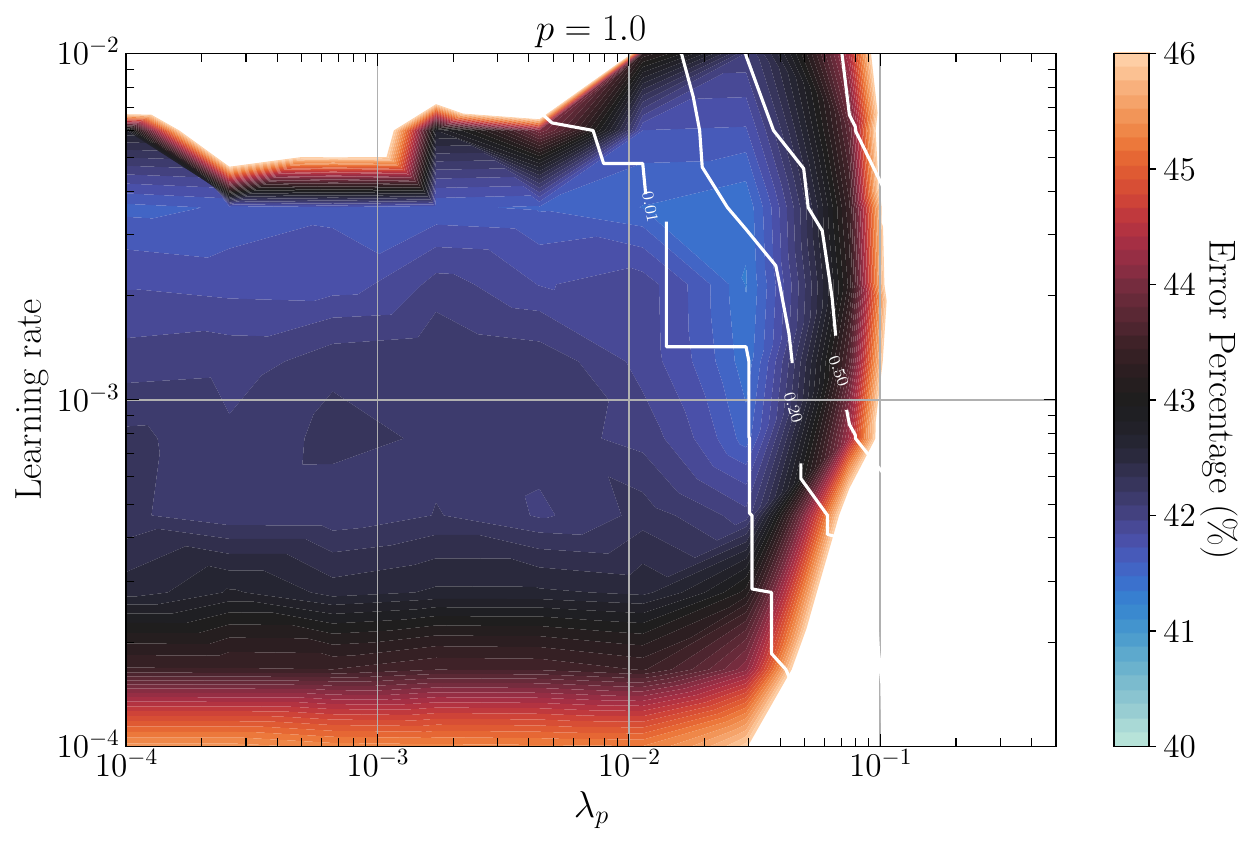}}
    \centerline{\includegraphics[width=0.3\textwidth]{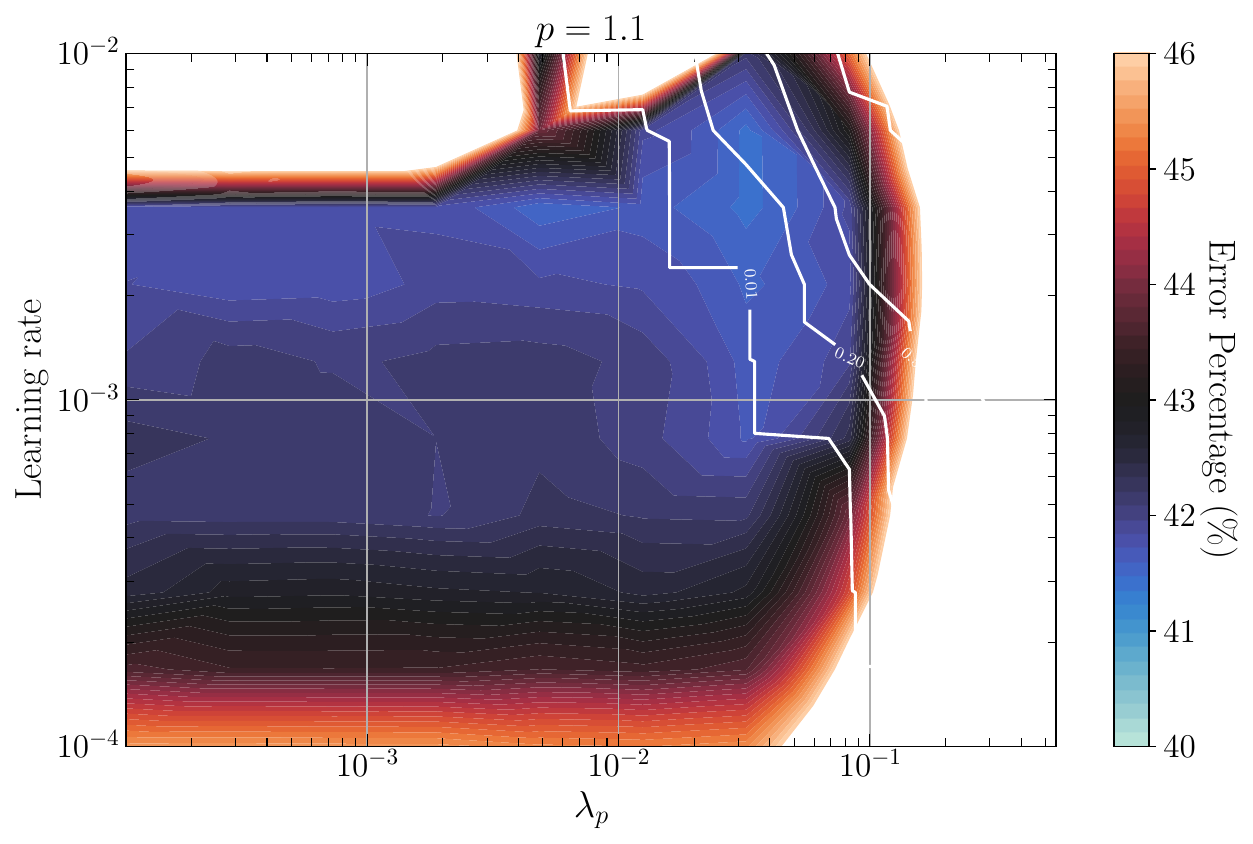}\includegraphics[width=0.3\textwidth]{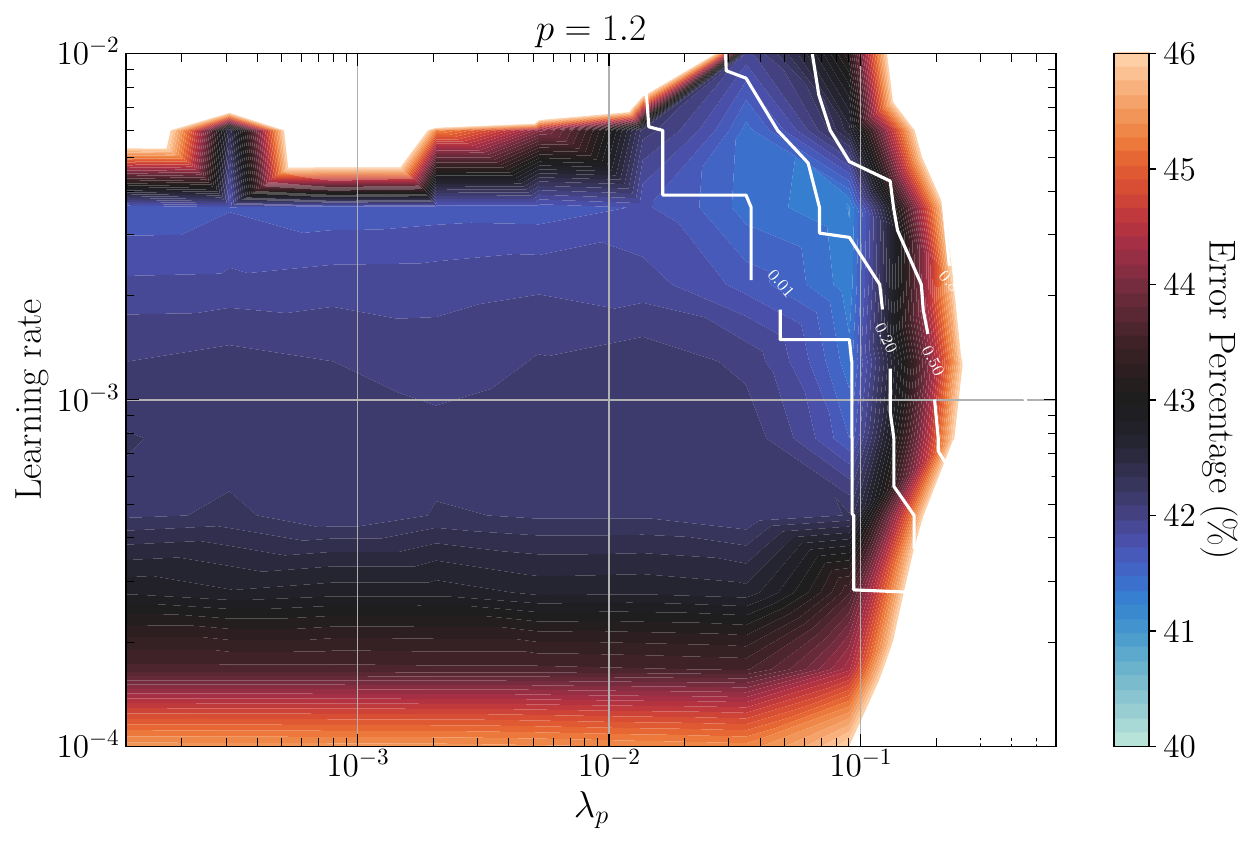}\includegraphics[width=0.3\textwidth]{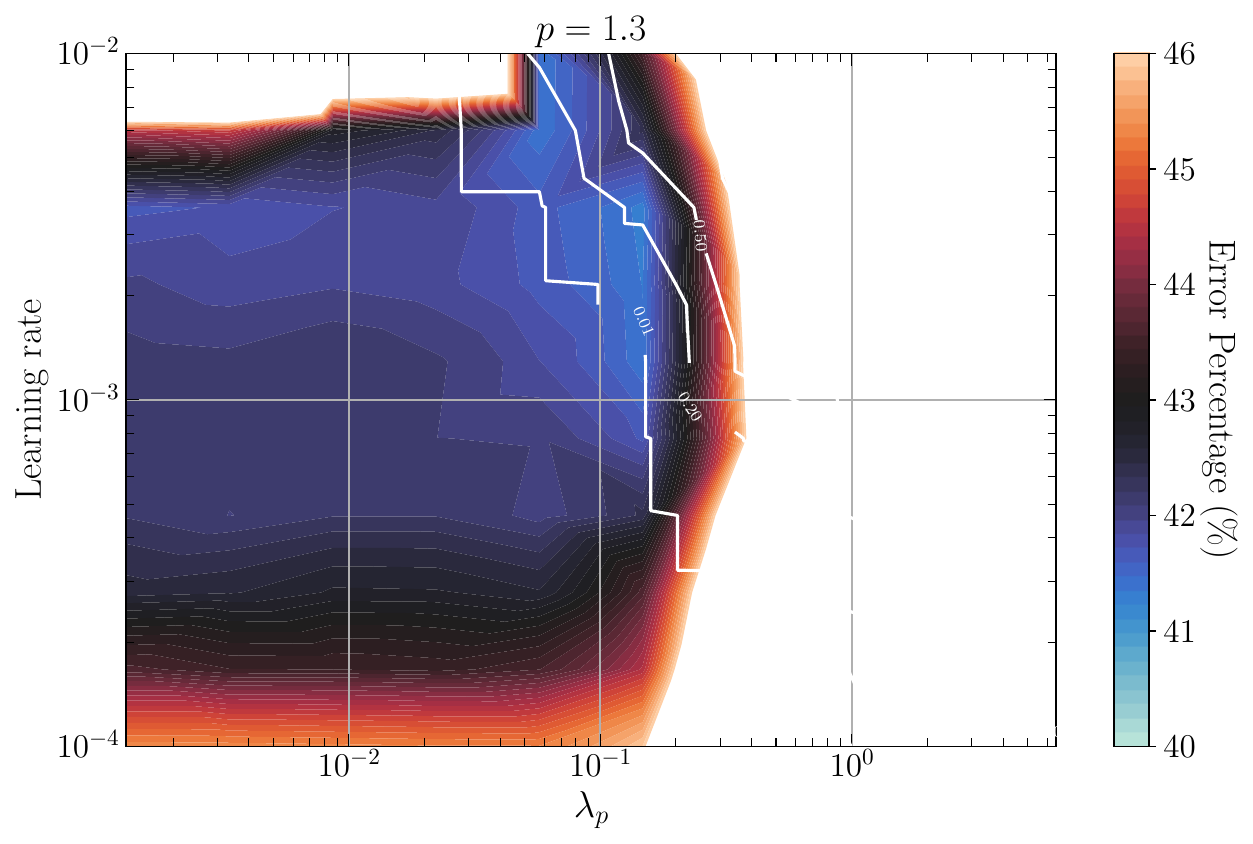}}
    \centerline{\includegraphics[width=0.3\textwidth]{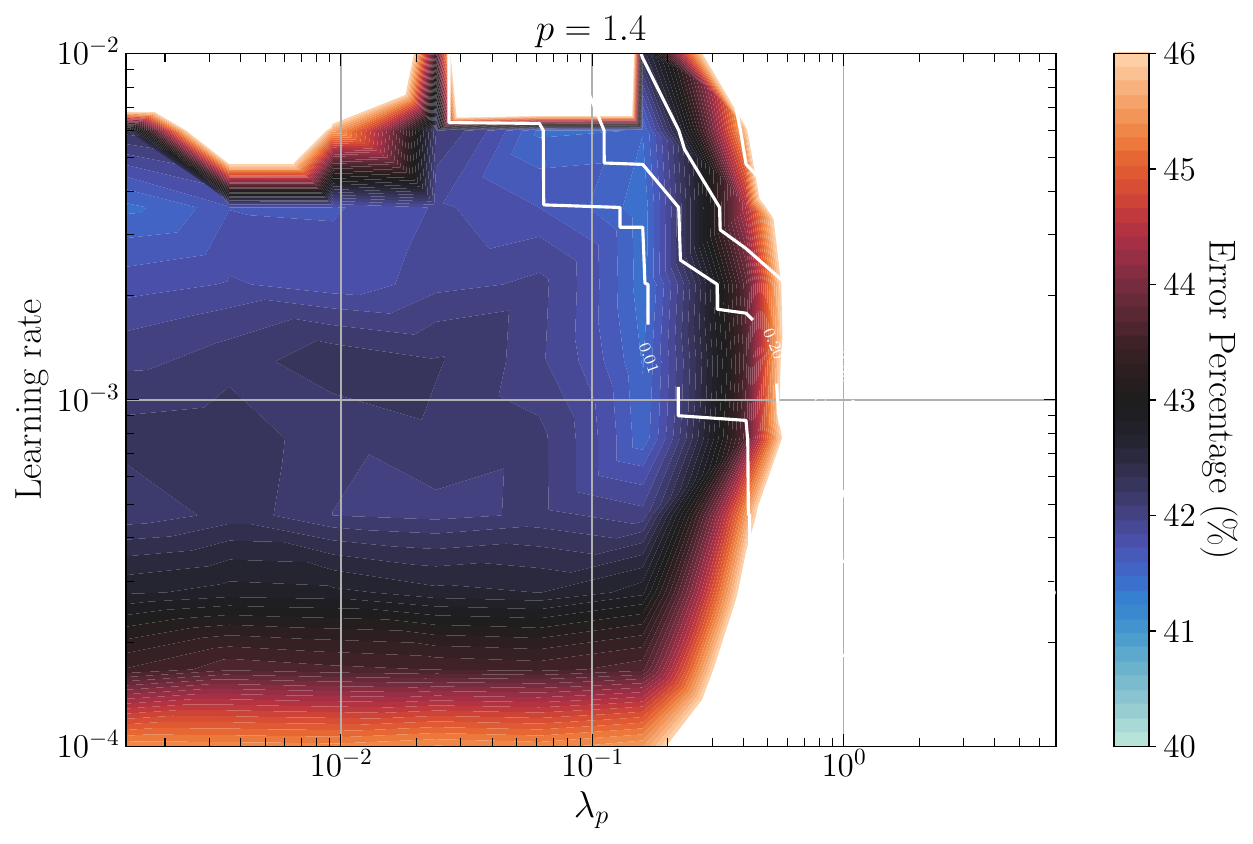}\includegraphics[width=0.3\textwidth]{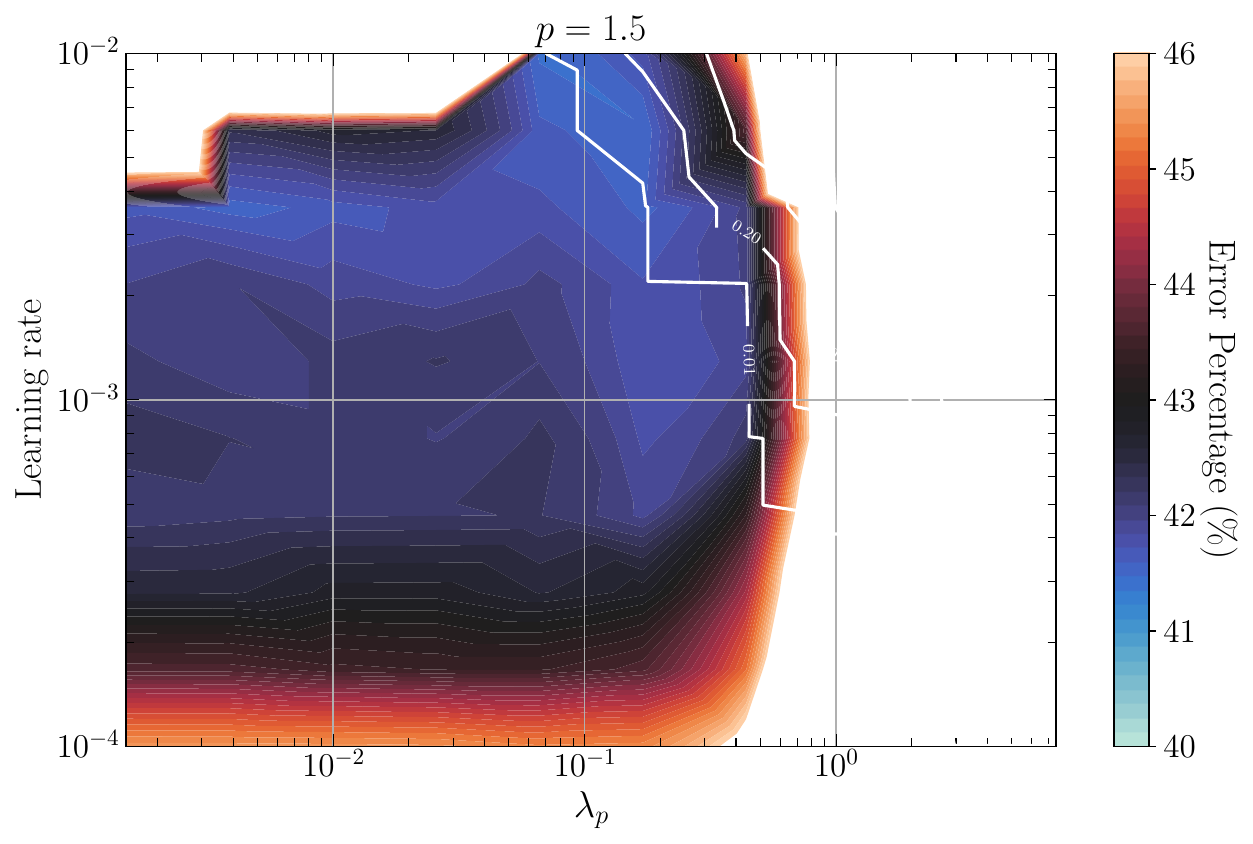}}
    
    \caption{
        Contours of validation accuracy after 5000 training iterations on the $\lambda_p$ vs. learning rate plane, for nanoGPT on Tiny Shakespeare. White contours represent the [0.01, 0.2, 0.4, 0.8] sparisty levlel. 
    }
    \label{fig:nanoGPT_contours}
    \end{center}
    \vskip -0.2in
\end{figure*}

\end{document}